\pdfoutput=1
\documentclass[nohyperref]{article}
\usepackage{graphicx}
\usepackage{xcolor}
\usepackage{microtype}
\usepackage{subfigure}
\usepackage{booktabs} 
\usepackage{esvect}
\usepackage{listings}
\usepackage{wrapfig}
\usepackage{amsfonts}
\definecolor{darkgreen}{RGB}{1,130,32}
\let\oldthanks\thanks
\renewcommand{\thanks}[1]{\let\footnotemark\relax\oldthanks{#1}}
\usepackage{caption}
\usepackage{fancyvrb}

\DeclareMathAlphabet{\mathmybb}{U}{bbold}{m}{n}
\newcommand{\1}{\mathmybb{1}}

\usepackage{undertilde}

\usepackage{titlesec}
\titlespacing*{\section}{0pt}{0.25\baselineskip}{0.2\baselineskip}
\titlespacing*{\subsection}{0pt}{0.2\baselineskip}{0.15\baselineskip}

\usepackage[accepted]{icml2025}

\usepackage{mystyle}

\usepackage[textsize=tiny]{todonotes}

\newcommand{\RN}[1]{%
  \textup{\uppercase\expandafter{\romannumeral#1}}%
}
\icmltitlerunning{Reward-Augmented Data Enhances Direct Preference Alignment of LLMs}

\begin{document}
\twocolumn[
\icmltitle{Reward-Augmented Data Enhances Direct Preference Alignment of LLMs}



\icmlsetsymbol{equal}{*}

\begin{icmlauthorlist}
\icmlauthor{Shenao Zhang}{equal,yyy}
\icmlauthor{Zhihan Liu}{equal,yyy}
\icmlauthor{Boyi Liu}{comp}
\icmlauthor{Yufeng Zhang}{comp}\\
\icmlauthor{Yingxiang Yang}{comp}
\icmlauthor{Yongfei Liu}{comp}
\icmlauthor{Liyu Chen}{comp}
\icmlauthor{Tao Sun}{comp}
\icmlauthor{Zhaoran Wang}{yyy}
\icmlaffiliation{yyy}{Northwestern University}
\icmlaffiliation{comp}{ByteDance}
\end{icmlauthorlist}

\icmlcorrespondingauthor{Shenao Zhang}{shenao@u.northwestern.edu}

\icmlkeywords{Machine Learning, ICML}

\vskip 0.3in
]



\printAffiliationsAndNotice{} 

\begin{abstract}
Preference alignment in Large Language Models (LLMs) has significantly improved their ability to adhere to human instructions and intentions. However, existing direct alignment algorithms primarily focus on relative preferences and often overlook the qualitative aspects of responses, despite having access to preference data that includes reward scores from judge models during AI feedback. Striving to maximize the implicit reward gap between the chosen and the slightly inferior rejected responses can cause overfitting and unnecessary unlearning of the high-quality rejected responses. The unawareness of the reward scores also drives the LLM to indiscriminately favor the low-quality chosen responses and fail to generalize to optimal responses that are sparse in data. To overcome these shortcomings, our study introduces reward-conditioned LLM policies that discern and learn from the entire spectrum of response quality within the dataset, helping extrapolate to more optimal regions. We propose an effective yet simple data relabeling method that conditions the preference pairs on quality scores to construct a reward-augmented dataset. The experiments across various benchmarks and diverse models demonstrate that our approach consistently boosts DPO by a considerable margin. Through comprehensive ablation studies, we demonstrate that our method not only maximizes the utility of preference data but also mitigates the issue of unlearning, demonstrating its broad effectiveness beyond mere data expansion. Our code is available at \url{https://github.com/shenao-zhang/reward-augmented-preference}.
\end{abstract}
\section{Introduction}
Reinforcement Learning from Human Feedback (RLHF) has recently seen remarkable success in aligning Large Language Models (LLMs) to follow instructions with human intentions. In this approach, AI-generated feedback serves as a stand-in for human preferences, assessing and ranking responses to prompts to construct a preference dataset. This dataset is then utilized in preference optimization algorithms to fine-tune LLMs. Among them, direct preference alignment \citep{rafailov2024direct,azar2023general,zhao2023slic,ethayarajh2024kto} that bypasses the need for an explicit reward model has garnered interest for its simplicity and cost efficiency. However, these algorithms mainly concern relative preferences and often overlook the quality of responses and their gaps, leading to limitations in their effectiveness.

Specifically, direct alignment algorithms such as DPO \citep{rafailov2024direct} focus on maximizing the implicit reward difference between accepted and rejected responses. This approach can lead to overfitting, as high-quality but rejected responses are unnecessarily unlearned \citep{adler2024nemotron}. Even worse, since the dataset provides only a sample estimate of true preferences, the rejected responses can actually be more aligned with human preferences than the accepted ones in expectation. Similarly, due to the unawareness of the responses' qualities, direct alignment will also result in the indiscriminate learning of the chosen responses, even when they are of low quality. As a result, the directly aligned LLMs often struggle to differentiate between responses of varying quality and fail to generalize effectively to more optimal or the highest-reward responses that are sparse in the preference data, which is another limitation. 

To address these issues, we propose learning reward-conditioned policies as a straightforward fix. By optimizing the LLM to generate responses conditioning on their qualities, the model is allowed to discern and leverage patterns within responses of varied quality. As a result, learning from both chosen and rejected responses alleviates the issue of unlearning high-quality rejected responses; distinguishing between varying-quality chosen responses alleviates the issue of indiscriminately accepting low-quality ones. By identifying common patterns in responses of similar quality and distinguishing them from those of differing quality, the LLM becomes more adept at generalizing to more optimal responses that are sparse in data.

With this motivation, we introduce an effective yet simple data relabeling method to construct reward-augmented datasets. We define a goal-conditioned reward using an indicator function that compares the goal reward with the actual quality score, such as the reward value given by the judge model during AI feedback. This allows us to relabel each preference pair, generating two new pairs conditioned on the reward goals of both the chosen and rejected responses. The resulting augmented dataset, which contains these newly conditioned pairs, can enhance the performance of existing direct alignment algorithms. Our method can be applied to any preference dataset and followed by off-the-shelf direct alignment algorithms to boost their performance. 

In experiments, we first apply our method on UltraFeedback \citep{cui2023ultrafeedback} and perform DPO \citep{rafailov2024direct} on this reward-augmented preference dataset by fine-tuning on various models, including Zephyr-7B-$\beta$ \citep{tunstall2023zephyr}, Mistral-7B-Instruct-v0.3 \citep{jiang2023mistral}, Qwen2-7B-Instruct \citep{yang2024qwen2}, Llama-3.1-8B-Instruct \citep{dubey2024llama}, Gemma-2-9B-It \citep{team2024gemma}, and SPPO \citep{wu2024self}. The results show that our method consistently boosts the performance of these models as well as their DPO models by a large margin on instruction-following benchmarks such as AlpacaEval 2.0 \citep{dubois2024length}, MT-Bench \citep{zheng2024judging}, and Arena-Hard-Auto \citep{li2024crowdsourced}. Our method also improves the average accuracy on a variety of academic benchmarks (GSM8K, GPQA, MUSR, TruthfulQA, BBH, and ARC). Moreover, our findings also demonstrate an improved utility of the preference data: a subsequent round of DPO using the reward-augmented data can still significantly enhance the model fine-tuned with DPO; relabeling the binarized preference dataset with the DPO implicit reward leads to further performance gains. Additional ablation studies also suggest that our method addresses the problem of unlearning and is superior not just due to the increased dataset size. When applied to on-policy data, our method enhances the DPO model, enabling it to surpass various baselines and achieve state-of-the-art performance on AlpacaEval 2.0.
\section{Background}
Consider an LLM $\pi\in\Delta_\mathcal{Y}^\mathcal{X}$ that takes the prompt $x\in\mathcal{X}$ as input and outputs the response $y\in\mathcal{Y}$, where $\mathcal{X}$ and $\mathcal{Y}$ are spaces of prompts and responses, respectively. Given $x\in\mathcal{X}$, a discrete probability distribution $\pi(\cdot\mid x)\in\Delta_\mathcal{Y}$ is generated, where $\Delta_\mathcal{Y}$ is the set of discrete distributions over $\mathcal{Y}$. We define the true human preference distribution as
\$
p^*(y_1\succ y_2 \mid x) &:= \EE_{h}\bigl[\1(h \text{ prefers } y_1 \text{ over } y_2 \text{ given } x)\bigr],
\$
where $h$ denotes the human rater and the expectation is over $h$ to account for the randomness of the human raters' choices. After pretraining and Supervised Fine-Tuning (SFT), Reinforcement Learning from Human or AI Feedback \citep{ouyang2022training,bai2022constitutional} is typically employed to enhance the ability of the language model to follow instructions with human preferences.

\paragraph{RL from AI Feedback (RLAIF).} The RLAIF framework involves two major steps: preference dataset construction with AI feedback and preference optimization. As a surrogate for human preference, AI feedback, including LLM-as-Judge \citep{zheng2024judging,cui2023ultrafeedback} and Reward-Model-as-Judge \citep{adler2024nemotron,dong2024rlhf}, can be used to rank responses and generate preference pairs. Specifically, consider the judge model $r(x, y): \mathcal{X}\times\mathcal{Y}\rightarrow \mathbb{R}$ that outputs a scalar reward value representing the quality of $y$ under $x$.
For each prompt $x\in\mathcal{X}$, two responses, $y_1$ and $y_2$, are independently sampled---either from the same reference model \citep{xiong2024iterative,wu2024self} or several different models \citep{starling2023,zhang2024self}. Then $r(x, y_1)$ and $r(x, y_2)$ are evaluated to determine the preferred response $y_w = \argmax_{y \in \{y_1, y_2\}} r(x, y)$ and dispreferred response $y_l = \argmin_{y \in \{y_1, y_2\}} r(x, y)$. By sampling responses and ranking them for a set of $N$ prompts, we get a preference dataset: $\mathcal{D}_N=\{(x^i, y_w^i, y_l^i)\}_{i=1}^N$. For the simplicity of our discussions, we assume that the reward function $r$ is bounded in $[0,r_{\text{max}}]$.

\paragraph{Direct Alignment from Preference.} The objective for the LLM $\pi\in\Delta_\mathcal{Y}^\mathcal{X}$ is to maximize the KL-regularized expected reward. Recent work \citep{azar2023general,zhao2023slic,tunstall2023zephyr,ethayarajh2024kto} proposed to align the LLM directly with the paired data by deriving the preference loss as a function of the LLM by the change of variables. Among them, the Direct Preference Optimization (DPO) \citep{rafailov2024direct} loss has the following form:
\$
&\mathcal{L}_{\text{DPO}}(\pi;\mathcal{D}_N) = -\EE_{(x, y_w, y_l)\sim\mathcal{D}_N}\biggl[\log\\
&\qquad\qquad\sigma\biggl(\beta\log\frac{\pi(y_w\mid x)}{\pi_{\text{ref}}(y_w\mid x)} - \beta\log\frac{\pi(y_l\mid x)}{\pi_{\text{ref}}(y_l\mid x)}\biggr)\biggr],
\$
where $\beta$ is a hyperparameter corresponding to the KL divergence regularization, $\sigma(\cdot)$ is the logistic function, and $\pi_{\text{ref}}$ is some reference LLM policy, such as the SFT model.
\section{Reward-Conditioning Addresses Limitations of Direct Preference Alignment}\label{sec_3}
\subsection{Limitations of Direct Alignment from Preference}
\label{sec_lim}
We will first demonstrate the limitations of vanilla direct alignment over the preference data.

\vspace{-0.15cm}
\paragraph{High-Quality Rejected Responses are Unnecessarily Unlearned.} The dataset $\mathcal{D}_N$ often contains preference pairs where the rejected response $y_l$ is only marginally worse than the chosen one $y_w$. Direct alignment algorithms, however,  primarily focus on relative preferences and are unaware of the responses' quality values and gaps. Striving to maximize the reparameterized reward gap between the chosen and rejected responses will risk overfitting and unnecessary ``unlearning", i.e., probability decrease, of high-quality responses, potentially diminishing the model's performance by discarding valuable alternatives. Furthermore, in such a finite data regime where only a sample estimate of the true preference is accessible, it can be very possible that $p^*(y_l \succ y_w\mid x) > 0.5$, i.e., $y_l$ is in fact more preferred than $y_w$ in expectation. This issue becomes even more pronounced when the preference data generated with the imperfect judge model is noisy.
\vspace{-0.2cm}
\begin{table}[htb]
\centering
\setlength{\tabcolsep}{4pt}
\begin{tabular}{@{}l|*{2}{>{\centering\arraybackslash}p{1.3cm}}@{}}
\toprule
response & $y_1$ & $y_2$ \\
\midrule
$r(x, y)$ & 9 & 8\\
$\mathcal{D}_{N=1}$ & \multicolumn{2}{c}{$\{y_1 > y_2\}$} \\
\midrule
$\pi^*(y\mid x)$ & 1 & 0 \\
\midrule
$\pi^*(y\mid x , g=9)$ & 1 & 0 \\
$\pi^*(y\mid x , g=8)$ & 0 & 1 \\
\bottomrule
\end{tabular}
\caption{High-quality rejected responses such as $y_2$ can be unnecessarily unlearned: $\pi^*(\cdot\mid x)$ deterministically generates $y_1$. Reward-conditioned policies learn both responses and are easier to generalize to $g=10$ with features extracted from $g=8$ and $9$.\vspace{-0.2cm}}
\label{tab_unlearn}
\end{table}

We illustrate this limitation with the example in Table \ref{tab_unlearn}, where we define the maximum reward $r_\text{max}$ as $10$. For $\mathcal{D}_{N=1}$ that contains a single preference pair\footnote{For simplicity, we write $(x, y_w, y_l)\in\mathcal{D}_N$ as $y_w \succ y_l$.} with reward $r(x, y_1)=9$ and $r(x, y_2)=8$, the optimal policy learned from $\mathcal{D}_{N=1}$ is $\pi^*(y_1\mid x)=1$. This causes the model to avoid generating $y_2$, a response of nearly equivalent quality.

\vspace{-0.15cm}
\paragraph{Low-Quality Chosen Responses are Indiscriminately Learned.} For a similar reason, direct alignment algorithms also  indiscriminately reinforce the chosen responses. As illustrated in Table \ref{tab_learn}, when $\mathcal{D}_{N=2}$ contains two preference pairs, where one of the chosen responses, $y_2$, is of low quality, $\pi^*$ still indiscriminately generates $y_2$ with an arbitrary probability $0\leq a\leq 1$, i.e., $\pi^*(y_2\mid x)=a$.

\vspace{-0.15cm}
\paragraph{Reward Sparsity.} Preference data often contains responses that, despite being preferred in pairwise comparisons, exhibit substantial variation in quality. As a result, the optimal responses---those associated with the highest reward value $r_{\text{max}}$---are sparse in the dataset. Since direct alignment algorithms do not account for these reward values, the trained model struggles to differentiate between responses of varying quality and fails to generalize effectively to the sparse optimal responses. 

\vspace{-0.2cm}
\begin{table}[htb]
\centering
\setlength{\tabcolsep}{4pt}
\begin{tabular}{@{}l|*{3}{>{\centering\arraybackslash}p{1.1cm}}@{}}
\toprule
response & $y_1$ & $y_2$ & $y_3$ \\
\midrule
$r(x, y)$ & 9 & 1 & 0 \\
$\mathcal{D}_{N=2}$ & \multicolumn{3}{c}{$\{y_1> y_3\, ,\, y_2> y_3\}$} \\
\midrule
$\pi^*(y\mid x)$ & $1-a$ & $a$ & 0 \\
\midrule
$\pi^*(y\mid x, g=9)$ & 1 & 0 & 0 \\
$\pi^*(y\mid x, g=1)$ & 0 & 1 & 0 \\
$\pi^*(y\mid x, g=0)$ & 0 & 0 & 1 \\
\bottomrule
\end{tabular}
\caption{Low-quality chosen responses such as $y_2$ can be learned: $\pi^*$ indiscriminately generates $y_1$ and $y_2$. Reward-conditioned policies distinguish the differences and learn the behaviors corresponding to different reward scores.\vspace{-0.4cm}}
\label{tab_learn}
\end{table}

\subsection{Reward-Conditioned Policies Learn from the Full Spectrum of Response Quality} 
A straightforward way to address the limitations of direct alignment algorithms---specifically, their inability to account for the quality of responses---is to optimize a reward-conditioned policy. In this approach, the LLM policy is trained to generate responses corresponding to different reward values, enabling it to become aware of and adapt to these reward distinctions. By doing so, the LLM not only learns the patterns associated with the preferred responses but also retains valuable information from the rejected ones, preventing the unlearning of high-quality rejected responses. For example, in Table \ref{tab_unlearn}, reward-conditioned policies (the last two rows) learn to generate both $y_1$ and $y_2$, instead of unlearning $y_2$. This reward-based conditioning also enhances the model’s ability to differentiate between responses of varying quality, even if both are preferred over a rejected alternative, as illustrated in Table \ref{tab_learn}. Besides, by extracting common patterns across responses with different quality levels, the LLM becomes more generalizable and is capable of generating the highest-quality responses with reward $r_{\text{max}}$ (e.g., $10$), which are often sparse in the training data.
\section{Method}
\begin{figure*}[htbp]
    \centering
    \includegraphics[width=\linewidth]{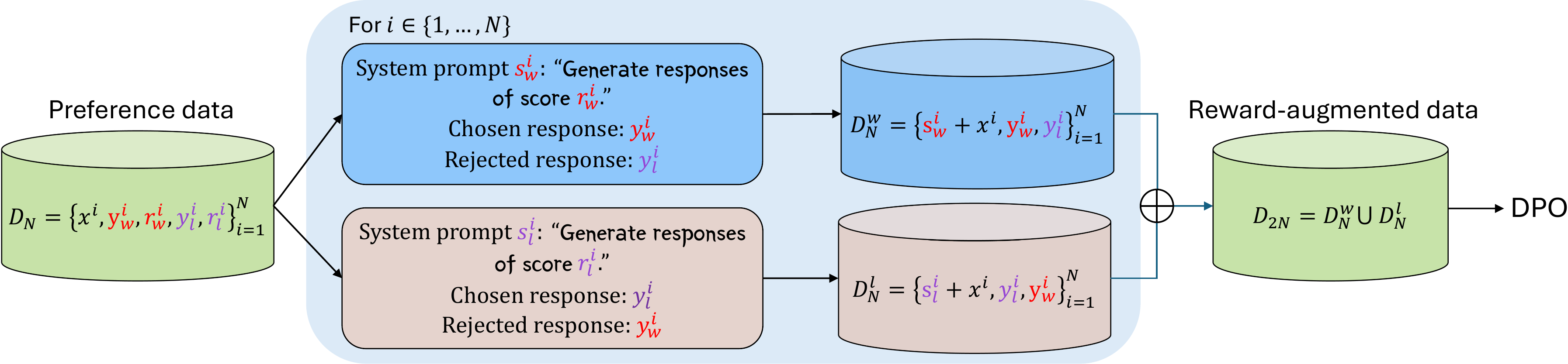}
  \caption{Illustration of our method: construction of reward-augmented preference datasets followed by direct alignment algorithms.}
    \label{fig_illu}
\end{figure*}
With the above motivation, we propose a data relabeling method that constructs a reward-augmented dataset by conditioning the preference pairs on the reward values given by the judge model $r$. Specifically, we define the goal-conditioned reward function $R(x, y,g) = - (g-r(x,y))^2$ as a function of the reward function $r$. The objective of the reward-conditioned policy $\pi(y\mid x ,g)$ is thus to minimize the square difference between the goal reward $g$ and the response reward $r(x, y)$, which is equivalent to maximizing the goal-conditioned reward $R(x, y,g)$, i.e.,
\vspace{-0.18cm}
\#
    \label{eq_obj}
&\min_\pi\EE_{g,x\sim\mathcal{D}_N, y\sim\pi(\cdot\mid x , g)}\bigl[(g-r(x,y))^2\bigr]  \notag\\
&\qquad\qquad= \max_\pi\EE_{g, x\sim\mathcal{D}, y\sim\pi(\cdot\mid x , g)}\bigl[R(x, y,g)\bigr].
\#
To optimize the RHS of \eqref{eq_obj}, we observe that under the new goal-conditioned reward metric $r$, we have
\vspace{-0.1cm}
\$
    R(x, y_w^i,g=r_w^i) &= 0 > R(x, y_l^i,g=r_w^i) = - (r_w^i-r_l^i)^2, \\
    R(x, y_l^i,g=r_l^i) &= 0 > R(x, y_w^i,g=r_l^i) = - (r_w^i-r_l^i)^2.
\$

\vspace{-0.2cm}
Thus, each pair can be relabeled to create two new preference pairs based on two distinct goals: when $g=r_w^i$, $y_w^i\succ y_l^i$; when $g=r_l^i$, $y_l^i\succ y_w^i$. Then any direct alignment algorithm can be applied to this new goal-conditioned preference dataset. Compared to fine-tuning on the original dataset $\mathcal{D}_N$, the model learns to capture not only desirable behaviors but also undesirable ones from the reward-augmented dataset. This approach helps identify patterns across high- and low-quality responses, enabling the LLMs to discern and learn from the entire spectrum of response quality and extrapolate to more optimal responses at inference time, by conditioning on higher reward goals.

We illustrate our method in Figure \ref{fig_illu}. For each preference pair with index $i$ in $\mathcal{D}_N$, two goals are defined, corresponding to the reward values of the chosen response $y_w^i$ and the rejected response $y_l^i$. Specifically, under the first goal $g=r_w^i$, the relabeled rewards are $R(x, y_w^i,g)=0$ and $R(x, y_l^i,g)= - (r_w^i-r_l^i)^2$. The original ranking of responses remains the same, except that the LLM is preference optimized conditioned on $g=r_w^i$. Similarly, under the second goal $g=r_l^i$, the relabeled rewards are $R(x, y_l^i,g)=0$ and $R(x, y_w^i,g)=-(r_w^i-r_l^i)^2$. Thus, the chosen and rejected responses are reversed as $y_l^i$ and $y_w^i$. By generating preference pairs conditioned on the goal reward for both the chosen and rejected responses, we obtain a reward-augmented dataset of size $2N$. Finally, this new dataset can be used with any direct alignment algorithm, such as DPO.

In this work, we implement the reward-conditioned policy $\pi(y\mid x ,g)$ as the LLM with a system prompt (or a prefix before the user prompt $x$ if system prompts are not supported by the LLM) such as ``generate responses of score $g$". At inference time, the LLM is conditioned on the optimal goal $g^\star=r_{\text{max}}$ that is the highest possible reward value, e.g., $g^\star=r_{\text{max}}=10$, to generate the responses.

Moreover, we provide the following theoretical guarantees for our method (see \ref{appendix_thm} for a formal description).
\begin{theorem}[Informal version]\label{thm:main} Let $J(\pi) = \mathbb{E}_{x\sim d_0,y\sim \pi(\cdot|x,g^\star)}\big[R(x,y,g^\star)\big]$ be the performance measure, where $R$ denotes the ground-truth goal-conditioned reward function and $g^\star$ denotes the optimal goal. Under mild assumptions, the policy $\widehat{\pi}$ optimized from the reward-augmented DPO with an SFT regularizer satisfies that with probability at least $1-\delta$,
    \begin{align}
        &J(\pi^*) - J(\widehat{\pi})\!\leq\! \sqrt{\frac{1}{N}}\cdot\notag\\
        &\qquad \bigg\{\frac{\sqrt{6}}{4} \big(1+\exp(B)\big)^2 \big(\big(C_{\mu_{\bar{\cD}}}(\cR; \pi^\star, \pi_\text{sft})\big)^2+1\big) \iota \notag\\&\qquad+  \mathbb{E}_{x\sim d_0}\Big[\mathrm{KL}\big(\pi^\star(\cdot|x, g^\star)\|\pi_{\mathrm{ref}}(\cdot|x, g^\star)\big) \Big]\bigg\},
    \end{align}
    where $\pi^* = \argmax_\pi J(\pi)$ and $\iota =  \sqrt{\log\left(\cN_{\varepsilon}(\cR,\|\cdot\|_{\infty})/\delta\right)}$ with $\varepsilon = (6\cdot(1+e^B)\cdot N)^{-1}$.
    Here, $N$ denotes the number of preference pairs in $\cD$, $B$ denotes the upper bound of the reward models, and the partial coverage coefficient $C_{\mu_{\bar{\cD}}}(\cR; \pi^\star,\pi_\text{sft})$ is defined in Assumption~\ref{as: coverage}.
\end{theorem}
The detailed proof is provided in \ref{app:proof}. The above theorem shows that our method attains global convergence to the optimal policy and the suboptimality decays at the order of $N^{-1/2}$ ($N$ denotes the size of the reward-augmented preference dataset), which provides a theoretical justification for the strong empirical performance of the introduced reward-augmented DPO. Unlike prior works on goal-conditioned RL with supervised learning \citep{yang2022rethinking,ghosh2019learning}, which typically establish weaker results such as local performance improvements or the optimization of a lower bound on $J(\pi)$, our analysis guarantees global convergence to the optimal policy. This distinction underscores the significance of integrating DPO-like methods with goal-conditioned approaches.

\section{Related Work}
\paragraph{Preference Dataset Construction.} In order for the LLMs to follow instructions and better align with human intents, it is common practice to build a preference dataset containing a set of prompts and a pair of responses for each prompt, whose qualities are ranked by humans \citep{ouyang2022training} or judge models \citep{bai2022constitutional}. A popular pipeline \citep{cui2023ultrafeedback,tunstall2023zephyr,wang2024far,ivison2023camels,starling2023} for constructing offline (i.e., fixed) datasets involves sampling off-policy responses from \textit{various} LLMs for each prompt in the hope to increase the response diversity. The preference data can also be generated online \citep{guo2024direct} or iteratively \citep{bai2022training,xu2023some,gulcehre2023reinforced,snorkelaipair,xiong2023gibbs,dong2024rlhf,calandriello2024human,rosset2024direct} by sampling and ranking on-policy responses from the training LLM. Recent works \citep{zhang2024self,cen2024value,xie2024exploratory} have also proposed systematically exploring the responses online and actively eliciting the preference. The proposed method in this paper is orthogonal to the construction ways of the preference data and can be applied to any dataset created either off-policy or on-policy.

\paragraph{Preference Optimization.} Preference optimization methods generally follow two approaches. The first involves learning a point-wise reward model, such as the Bradley-Terry model, and using RL algorithms like PPO \citep{schulman2017proximal,zheng2023secrets,xu2024dpo} or REINFORCE \citep{williams1992simple,li2023remax,ahmadian2024back}, to maximize the KL-regularized expected reward. The second approach is direct alignment \citep{rafailov2024direct,azar2023general,zhao2023slic,ethayarajh2024kto,liu2024provably}, which gets rid of a separate reward model that is computationally costly to train. In this work, we mainly focus on the limitations of direct alignment algorithms, particularly their unawareness of the quality aspects of responses. For PPO-style alignment algorithms that fit and maximize an explicit reward, preference data is only used to learn the reward model, and policy training is performed in an online manner, where responses are sampled from the LLM and their reward values directly play a role during the RL optimization. This avoids drawbacks inherent to direct alignment methods, as detailed in Section \ref{sec_lim}.

\begin{figure*}[htb]
\centering
\subfigure[AlpacaEval 2.0 results. Left: Length-Controlled (LC) win rates. Right: Win rates.]{
\label{transfer_pac}
\includegraphics[width=0.35\linewidth]{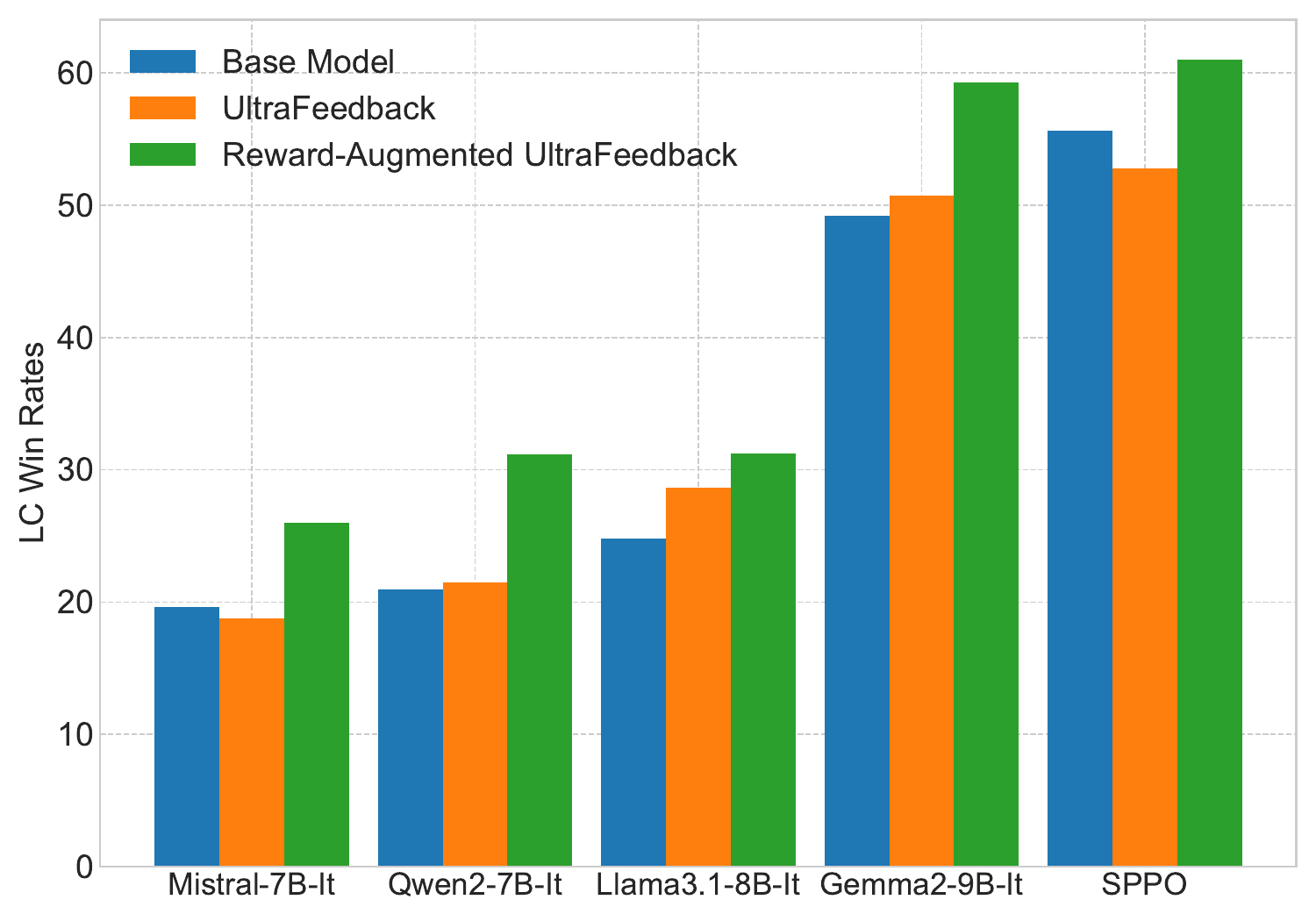}
\hspace{0.05in}
\includegraphics[width=0.35\linewidth]{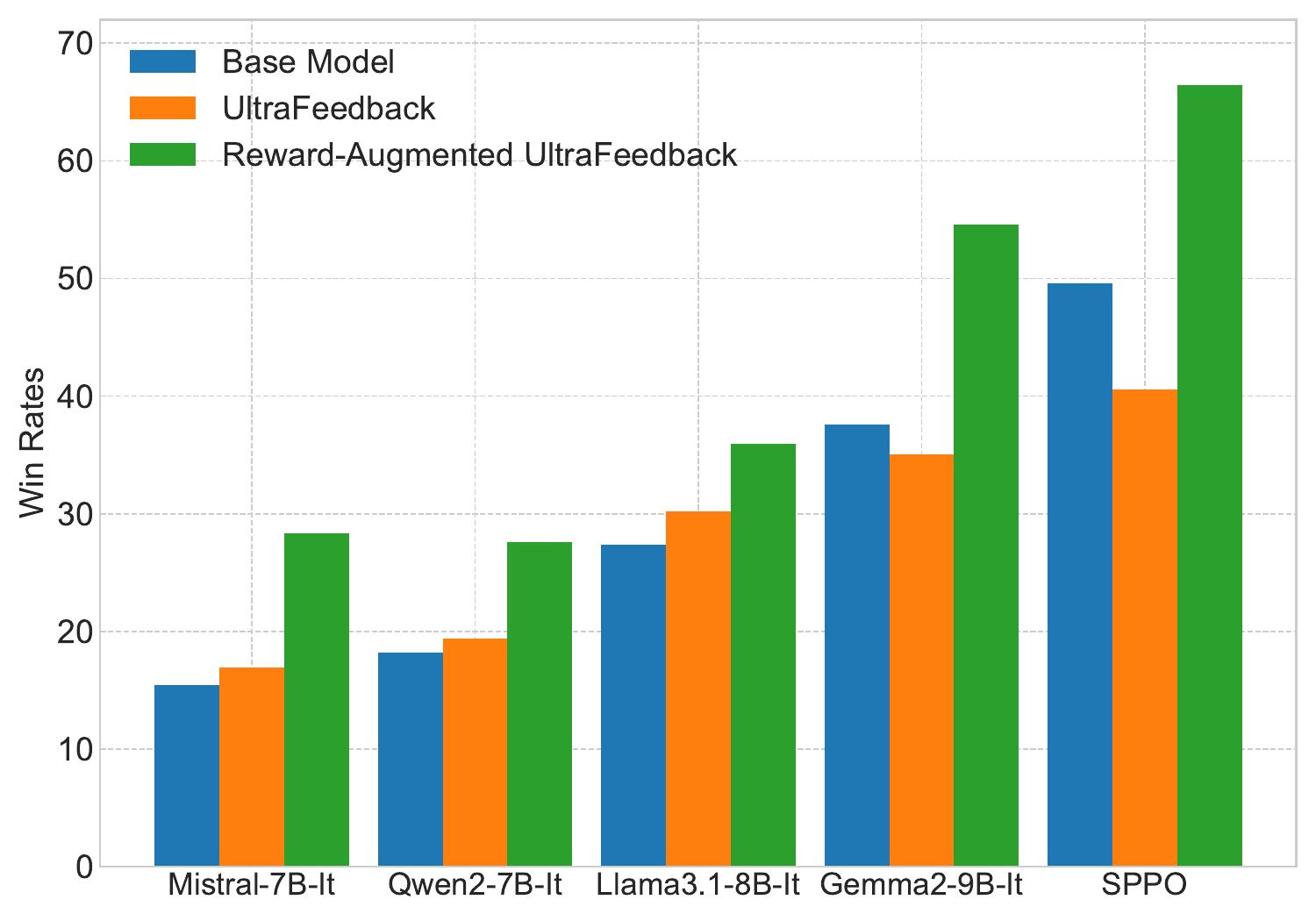}
}
\subfigure[MT-Bench average score.]{
\centering
\label{complex}
\includegraphics[width=0.35\linewidth]{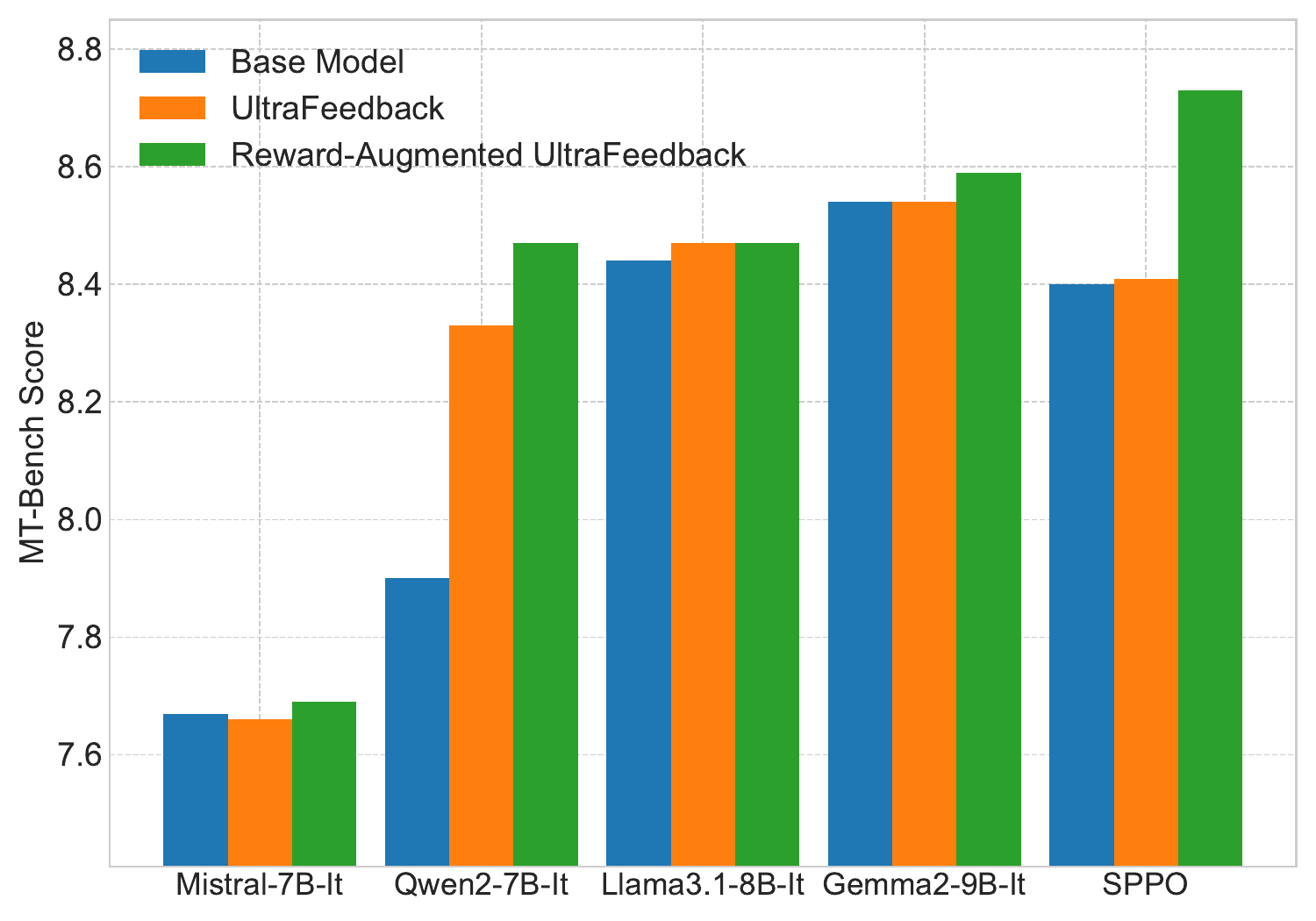}
}
\subfigure[Arena-Hard-Auto score.]{
\centering
\label{sparse}
\includegraphics[width=0.35\linewidth]{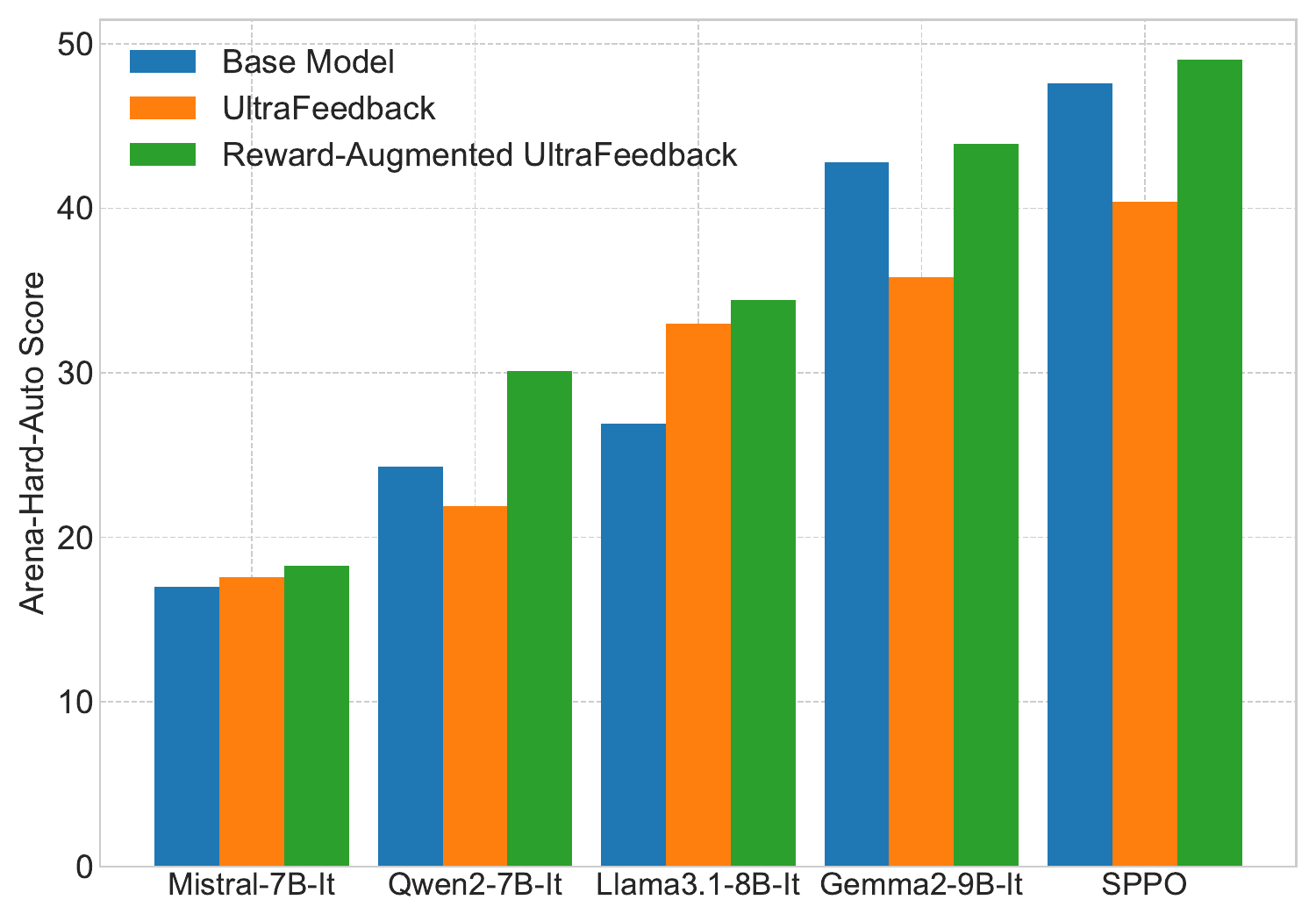}
}
\vspace{-0.1cm}
\caption{Instruction-following performance of the base models, the models trained with DPO on UltraFeedback, and the models trained with DPO on reward-augmented UltraFeedback on AlpacaEval 2.0, MT-Bench, and Arena-Hard-Auto benchmarks. Our method demonstrates considerable improvements consistently across all benchmarks. The complete table is deferred to Appendix \ref{sec_full}.\vspace{-0.1cm}}
\label{lmasjudge}
\end{figure*}

\paragraph{Conditional LLM Fine-Tuning.} Conditioning LLMs during training has proven effective for aligning responses with specific human objectives. SteerLM \citep{dong2023steerlm,wang2023helpsteer} extends SFT by conditioning the LLM on the multi-dimensional annotated attributes in data, such as humor and toxicity, in order to steer model responses with user customizability. Directional Preference Alignment (DPA) \citep{wang2024arithmetic} proposed a variant of rejection sampling fine-tuning \citep{yuan2023scaling,dong2023raft} that conditions on the direction of the multi-objective reward, i.e., a user-dependent linear combination of the reward attributes (helpfulness and verbosity in their experiments), that represents diverse preference objectives. These methods aim to train a single LLM that can flexibly adjust to various user preference profiles. On the contrary, our method targets the limitations of direct alignment algorithms by introducing reward-augmented relabeling. This also differs from Conditioned-RLFT \citep{wang2023openchat}, which supervised learns class-conditioned policies with data source information, and \cite{cai2023ulma,chen2024noise} that focus on unified alignment with binarized or reward datasets. Reward-aware Preference Optimization (RPO), introduced in Nemotron-4 \citep{adler2024nemotron}, attempts to approximate the reward gap using the implicit reward and is motivated to resolve the unlearning issues of DPO, which our work also addresses. However, we show that more limitations beyond unlearning can be simply fixed with reward-conditioned LLMs and propose an easy-to-implement data relabeling method without algorithm changes. \cite{pang2024iterative} addressed DPO’s tendency to reduce the probability of the chosen response by incorporating an NLL loss. In contrast, our work focuses on a different limitation of DPO—its tendency to overlook qualitative aspects of responses—and proposes a data relabeling approach that requires no algorithm changes. It also differs from conditional sequence modeling based on SFT \citep{chen2021decision,lloret2024towards}. Due to the lack of textual feedback in UF, we empirically compare with the reward feedback variants of \cite{lloret2024towards}, including SteerLM and DPA. Similar to our work, \cite{kim2024margin} also investigates how preference optimization can overlook qualitative aspects of responses. However, their focus is on overfitting to preference data, and they propose incorporating quality margins into the optimization objective. In contrast, our approach does not involve algorithmic modifications, but rather directly targets the limitations caused by the unawareness of the
reward scores. Our work also differs from \cite{park2024disentangling}, which introduces a constraint-based regularization term specifically aimed at mitigating verbosity bias.

\section{Experiments}
\subsection{Reward-Augmented Data Boosts DPO Performance}
We begin by conducting experiments to demonstrate that applying our method on DPO to fixed offline preference datasets leads to consistent performance improvements. Our setups are listed as follows.

\vspace{-0.2cm}
\paragraph{Setup.} We adopt the UltraFeedback \citep{cui2023ultrafeedback} preference dataset for this experiment. Specifically, UltraFeedback contains reward values scored by GPT-4 (LLM-as-Judge), which are ranged between $1$ and $10$ for each of the preference pairs. Our method constructs reward-augmented data by conditioning on these judge values. We choose to fine-tune on various open-weight LLMs, including Mistral-7B-Instruct-v0.3, Qwen2-7B-Instruct, Llama-3.1-8B-Instruct, Gemma-2-9B-It, and SPPO (fine-tuned from Gemma2-9B-It). The hyperparameters and prompts that we use are listed in Appendix \ref{app_setup}. 

\vspace{-0.2cm}
\paragraph{Results.} We first report the performance of the trained models on instruction-following benchmarks that use strong LLMs such as GPT-4 as judges, including AlpacaEval 2.0 \citep{dubois2024length}, MT-Bench \citep{zheng2024judging}, and Arena-Hard-Auto \citep{li2024crowdsourced}. The results are shown in Figure \ref{lmasjudge}.

\begin{table*}[ht]
\centering
\begin{tabular}{l|ccccccc}
\hline
Model     & \begin{tabular}[c]{@{}c@{}}GSM8K\\(8-s CoT)\end{tabular} & \begin{tabular}[c]{@{}c@{}}GPQA\\ (0-s)\end{tabular} & \begin{tabular}[c]{@{}c@{}}MUSR\\ (0-s)\end{tabular} & \begin{tabular}[c]{@{}c@{}}TruthfulQA\\ (0-s)\end{tabular} & \begin{tabular}[c]{@{}c@{}}BBH\\ (3-s)\end{tabular} & \begin{tabular}[c]{@{}c@{}}ARC\\ (25-s)\end{tabular} & Average \\ \hline
Mistral-7B-Instruct-v0.3  & 52.39 & \textbf{30.62} & \textbf{47.35} & 59.71 & 46.64 & 58.53 & 49.21 \\
+DPO (UltraFeedback)  & \textbf{53.22} & 28.94 & 47.35 & 64.74 & \textbf{47.46} & 60.32 & \textbf{50.34} \\
+DPO (Reward-Augmented)  & 51.86 & 28.02 & 46.56 & \textbf{65.90} & 46.36 & \textbf{61.60} & 50.05 \\
\hline
Qwen2-7B-Instruct  & 78.24 & 32.80 & 44.58 & 57.31 & \textbf{55.20} & 53.75 & 53.65 \\
+DPO (UltraFeedback)  & 78.17 & 32.80 & 44.31 & \textbf{58.91} & 54.49 & 53.75 & 53.74 \\
+DPO (Reward-Augmented)  & \textbf{81.05} & \textbf{32.97} & \textbf{45.77} & 57.99 & 54.94 & \textbf{54.52} & \textbf{54.54} \\
\hline
Llama-3.1-8B-Instruct  & 76.72 & \textbf{33.89} & 39.95 & 54.00 & 50.74 & 55.38 & 51.78 \\
+DPO (UltraFeedback)  & 78.47 & 33.72 & 43.39 & 56.61 & 51.31 & \textbf{57.51} & 53.50 \\
+DPO (Reward-Augmented)  & \textbf{78.77} & 32.55 & \textbf{43.52} & \textbf{63.32} & \textbf{51.57} & 56.48 & \textbf{54.37} \\
\hline
Gemma-2-9B-It  & 81.35 & \textbf{36.33} & 46.03 & 60.15 & 59.42 & 64.85 & 58.02 \\
+DPO (UltraFeedback)  & 83.32 & 34.14 & 46.56 & 65.12 & 59.78 & \textbf{66.41} & 59.22 \\
+DPO (Reward-Augmented)  & \textbf{83.62} & 35.74 & \textbf{48.15} & \textbf{65.27} & \textbf{59.82} & 65.87 & \textbf{59.75} \\
\hline
SPPO  & 79.83 & 35.91 & 44.97 & 62.56 & \textbf{59.61} & 63.74 & 57.77 \\
+DPO (UltraFeedback)  & \textbf{81.73} & 33.64 & 45.50 & 65.72 & 59.16 & \textbf{66.89} & 58.77 \\
+DPO (Reward-Augmented)  & 80.67 & \textbf{36.16} & \textbf{48.68} & \textbf{67.39} & 58.88 & 65.53 & \textbf{59.55}\\
\hline
\end{tabular}
\caption{Performance comparison between the LLMs after DPO on UltraFeedback, on reward-augmented UltraFeedback, and their base models on academic multi-choice QA benchmarks in standard zero-shot, few-shot, and CoT settings. Here, n-s refers to n-shot, the \textbf{bold} texts represent the best results in each family of models.\vspace{-0.15cm}}
\label{tab_academic}
\end{table*}

Across all instruction-following benchmarks, we observe that LLMs fine-tuned with DPO on the proposed reward-augmented data consistently outperform both their base models and those fine-tuned using DPO on the original UltraFeedback dataset by a considerable margin. Notably, direct alignment with the original preference data can sometimes degrade the performance of base models on specific benchmarks, such as Arena-Hard-Auto, which involves more complex reasoning tasks. In contrast, alignment using the reward-augmented data consistently yields superior results not only due to the improved style format gained from performing DPO on UltraFeedback.

Besides, we also evaluate the models on six academic multi-choice question-answering benchmarks, including GSM8K \citep{cobbe2021training}, GPQA \citep{rein2023gpqa}, MUSR \citep{sprague2023musr}, TruthfulQA \citep{lin2021truthfulqa}, BBH \citep{suzgun2022challenging}, and ARC Challenge \citep{clark2018think}. To better reflect the capabilities of LLMs, we adopt various settings for these benchmarks, including zero-shot, few-shot, and few-shot Chain-of-Thought (CoT). The results are shown in Table \ref{tab_academic}. 

It can be observed that performing DPO on the reward-augmented preference data leads to better average academic scores for most families of models compared to models fine-tuned on the original UltraFeedback dataset and the base models. Besides, we didn't observe severe alignment tax phenomenons \citep{askell2021general,noukhovitch2024language,li2024multi} after DPO, and our method is able to improve the base models on most of the benchmarks.

\subsection{Ablation Studies}
\vspace{-0.1cm}
\paragraph{Our Method Improves the Utility of Preference Data.} We provide two pieces of evidence that our method can get more juice out of the preference data compared to directly applying DPO. Firstly, we evaluate SPPO \citep{wu2024self} fine-tuned with DPO on UltraFeedback (UF). The results are shown in Table \ref{tab_sppo}. Since the SPPO model is already trained on UltraFeedback from Gemma-2-9B-It, an additional round of DPO training with the same data significantly degrades its performance. In contrast, performing DPO on Reward-Augmented (RA) UltraFeedback results in substantial performance gains for SPPO, indicating that our method enhances the utility of the preference data.

\begin{table}[H]
\centering
\begin{tabular}{l|l|l|l|l}
\hline
                        & LC WR & \hspace{0.1cm}WR    & MT   & \hspace{-0.1cm}Arena \\ \hline
SPPO                    & 55.60 & 49.61 & 8.40 & 47.6  \\
+DPO (UF)               & 52.75 & 40.58 & 8.41 & 40.4  \\
+DPO (RA)               & \textbf{60.97} & \textbf{66.41} & \textbf{8.73} & \textbf{49.0}  \\ \hline
\end{tabular}
\vspace{-0.2cm}
\caption{SPPO can be improved with DPO by performing reward augmentation on the same data.\vspace{-0.3cm}}
\label{tab_sppo}
\end{table}

The second evidence is that after DPO, the implicit reward can be used to relabel and augment the same preference data. Specifically, after training Qwen2-7B-Instruct with DPO on UltraFeedback, we leverage the resulting model $\pi_{\text{DPO}}$ to calculate the implicit reward for each prompt $x$ and response $y$, i.e., $\hat{r} = \beta(\log\pi_{\text{DPO}}(y\mid x) - \log\pi_{\text{Qwen}}(y\mid x))$. Then we perform DPO on Qwen2-7B-Instruct using the Implicit-Reward-Augmented (IRA) UltraFeedback. The results are shown in Table \ref{tab_ira}. We observe that augmenting the data with the implicit reward from the DPO (UF) model leads to superior performance even compared to augmenting the data with reward scores from the LLM judge, i.e., DPO (RA). This result highlights that DPO does not fully exploit the potential of the data. Moreover, this ablation demonstrates that our method is compatible with binarized preference datasets that only contain chosen and rejected response pairs, bypassing the need of judge models.
\begin{table}[H]
\centering
\begin{tabular}{l|l|l|l|l}
\hline
                        & LC WR & \hspace{0.1cm}WR    & MT   & \hspace{-0.1cm}Arena \\ \hline
Qwen2-7B-It          & 20.93 & 18.22 & 7.90 & 24.3  \\ 
+DPO (UF)               & 21.46 & 19.35 & 8.33 & 21.9  \\
+DPO (RA)               & 31.17 & 27.58 & 8.47 & \textbf{30.1}  \\ 
+DPO (IRA)               & \textbf{32.61} & \textbf{29.15} & \textbf{8.49} & 28.3  \\ \hline
\end{tabular}
\caption{A second round of DPO on the reward-augmented data, i.e., DPO (IRA), relabeled with the implicit reward from the DPO model at the first round, i.e., DPO (UF), significantly improves it. Our method helps get more juice out of the \textit{binarized} (i.e., without judge model rewards) preference data.\vspace{-0.15cm}}
\label{tab_ira}
\end{table}

\vspace{-0.2cm}
\paragraph{Reward-Augmented Data is Superior Not Just Due to Its Increased Size.} In this part, we show that the success of our method is not merely due to the increased size of the training dataset. To illustrate this, we perform DPO on the
\begin{table}[H]
\centering
\begin{tabular}{l|l|l|l|l}
\hline
                        & LC WR & \hspace{0.1cm}WR    & MT   & \hspace{-0.1cm}Arena \\ \hline
Qwen2-7B-It          & 20.93 & 18.22 & 7.90 & 24.3  \\ 
+DPO (UF)               & 21.46 & 19.35 & 8.33 & 21.9  \\ 
+DPO (RA)               & \textbf{31.17} & 27.58 & \textbf{8.47} & \textbf{30.1}  \\ 
+DPO (Half RA)     & 29.56 & \textbf{28.30} & 8.33 & 26.9 \\ \hline
Gemma-2-9B-It         & 49.20 & 37.58 & 8.54 & 42.8  \\ 
+DPO (UF)               & 50.70 & 35.02 & 8.54 & 35.8  \\ 
+DPO (RA)               & \textbf{59.27} & \textbf{54.56} & 8.59 & \textbf{43.9}  \\ 
+DPO (Half RA)     & 53.12 & 43.74 & \textbf{8.66} & 41.3 \\ \hline
\end{tabular}
\caption{DPO trained on only half of the data with reward augmentation outperforms the baseline.\vspace{-0.2cm}}
\label{tab_half}
\end{table} 
dataset where reward augmentation is applied to the first half of the UltraFeedback data, which we denote as DPO (Half RA). By doing so, the reward-augmented data is of the same size as the original dataset, but with only half of the prompts and the corresponding responses being utilized. It can be observed from Table \ref{tab_half} that DPO (Half RA) outperforms fine-tuning on the whole UltraFeedback (UF) by a large margin and achieves comparable performance to applying reward augmentation across the entire UF dataset, which is denoted as DPO (RA).

\paragraph{Reward-Augmented Data Mitigates the Unlearning Issue.} We first demonstrate that DPO suffers from the limitation of unnecessarily unlearning high-quality rejected responses, as discussed in Section \ref{sec_lim}. Specifically, on the test set of UltraFeedback, we calculate the log probability of each rejected response for the Qwen2-7B-Instruct model, its DPO (UF) model, and our method DPO (RA). In Fig. \ref{fig_unlearn}, we plot the expected log probability for rejected responses
\begin{figure}[hb]
    \centering
    \includegraphics[width=0.35\textwidth]{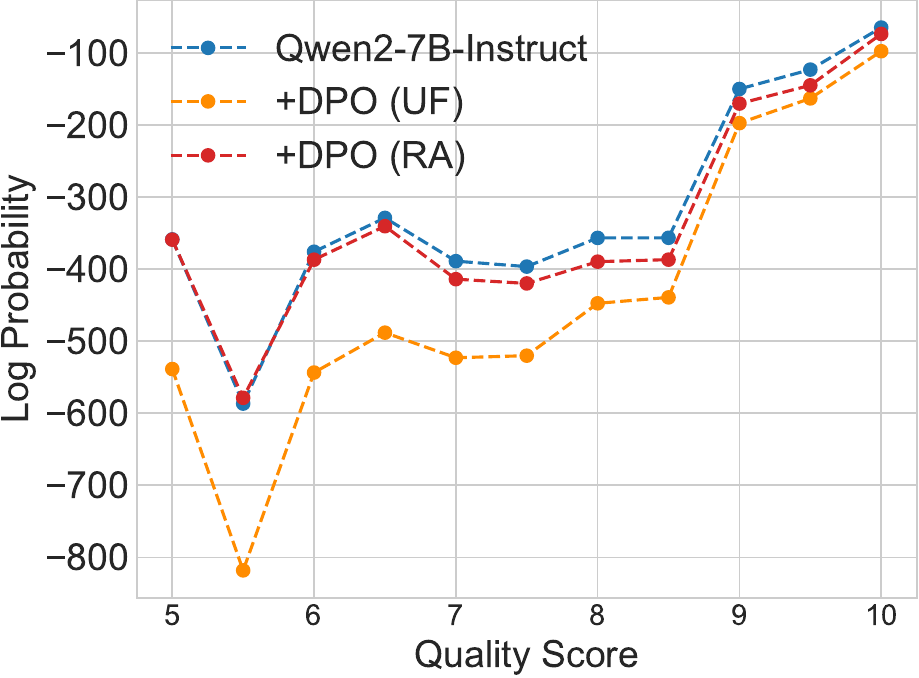}
    \vspace{-0.25cm}
    \caption{Our method helps mitigate the unlearning issue of DPO.}
    \label{fig_unlearn}
\end{figure}

whose reward scores $\geq 5$. We find that DPO substantially decreases the log probability of these high-quality rejected responses, confirming that the unlearning issue arises in practice. In contrast, our method alleviates this issue, although the probability is still slightly lower than the base model, which is proven to be DPO's feature \citep{rafailov2024r,zhang2024self,xu2024dpo}.

\vspace{-0.1cm}
\paragraph{Conditioning on Multi-Attribute Rewards Enables SOTA Models.} In previous parts, our method is implemented by conditioning on the scalar reward values given by the judge models, either LLMs or reward models. We find that our approach is generalizable to settings of multi-dimensional rewards that correspond to different attributes, such as helpfulness and truthfulness. Specifically, we follow the setting from last part to construct the preference dataset by applying the ArmoRM reward model on the on-policy responses generated by Llama-3-8B-Instruct. Since ArmoRM is a multi-objective model that not only gives a scalar reward value but also predicts human-interpretable fine-grained attributes, we first select $5$ attributes (namely complexity, instruction following, honesty, helpfulness, and intelligence depth) that have the highest average coefficients on the UltraFeedback data. Then we relabel the data by conditioning on the $5$-dim reward and follow the implementation of using ArmoRM described in the last part. The resulting model achieves state-of-the-art within the Llama-3-8B-Instruct model family, surpassing the strong baselines including SimPO \citep{meng2024simpo} that is trained also on on-policy data ranked by ArmoRM, and OpenChat \citep{wang2023openchat} fine-tuned with Conditioned-RLFT from the same Llama-3-8B-Instruct model.

\begin{table}[H]
\centering
\begin{tabular}{l|ccc}
\hline
&\begin{tabular}{@{}c@{}}LC Win Rate\end{tabular} & \begin{tabular}{@{}c@{}}Win Rate \end{tabular} & \begin{tabular}{@{}c@{}}Avg. Len.\end{tabular} \\ \hline
Ours & \textbf{56.57} & \textbf{52.19}  & 1840 \\ \hline
SimPO & 53.70 & 47.50  & 1777 \\ \hline
OpenChat & 17.48 & 11.36  & 1362 \\ \hline
\end{tabular}
\vspace{-0.1cm}
\caption{Our method trained with DPO achieves SOTA when conditioning on $5$-dim rewards.\vspace{-0.3cm}}
\label{tab_armorm}
\end{table}

\vspace{-0.1cm}
\paragraph{Comparison with Conditional Fine-Tuning Baselines.} We further compared with additional conditional post-training baselines on the offline UltraFeedback dataset (i.e., without on-policy data), including DPA \citep{wang2024arithmetic}, SteerLM \citep{dong2023steerlm}, and (Info)NCA \citep{chen2024noise}. Since both baselines aim to optimize a user-controllable attribute-conditioned LLM that is optimal under diverse preference profiles with different coefficients of the reward's attributes, in Figure \ref{fig_baselines}, we plot the win rates of these methods under varying preference profiles, such as adjusting verbosity preferences as considered in \cite{wang2024arithmetic}. Fine-tuned from Zephyr-SFT, our method achieves the best AlpacaEval 2.0 win rate. We defer more ablations and baseline comparisons to Appendix \ref{app_more_abl}.

\begin{figure}[H]
    \centering
    \includegraphics[width=0.35\textwidth]{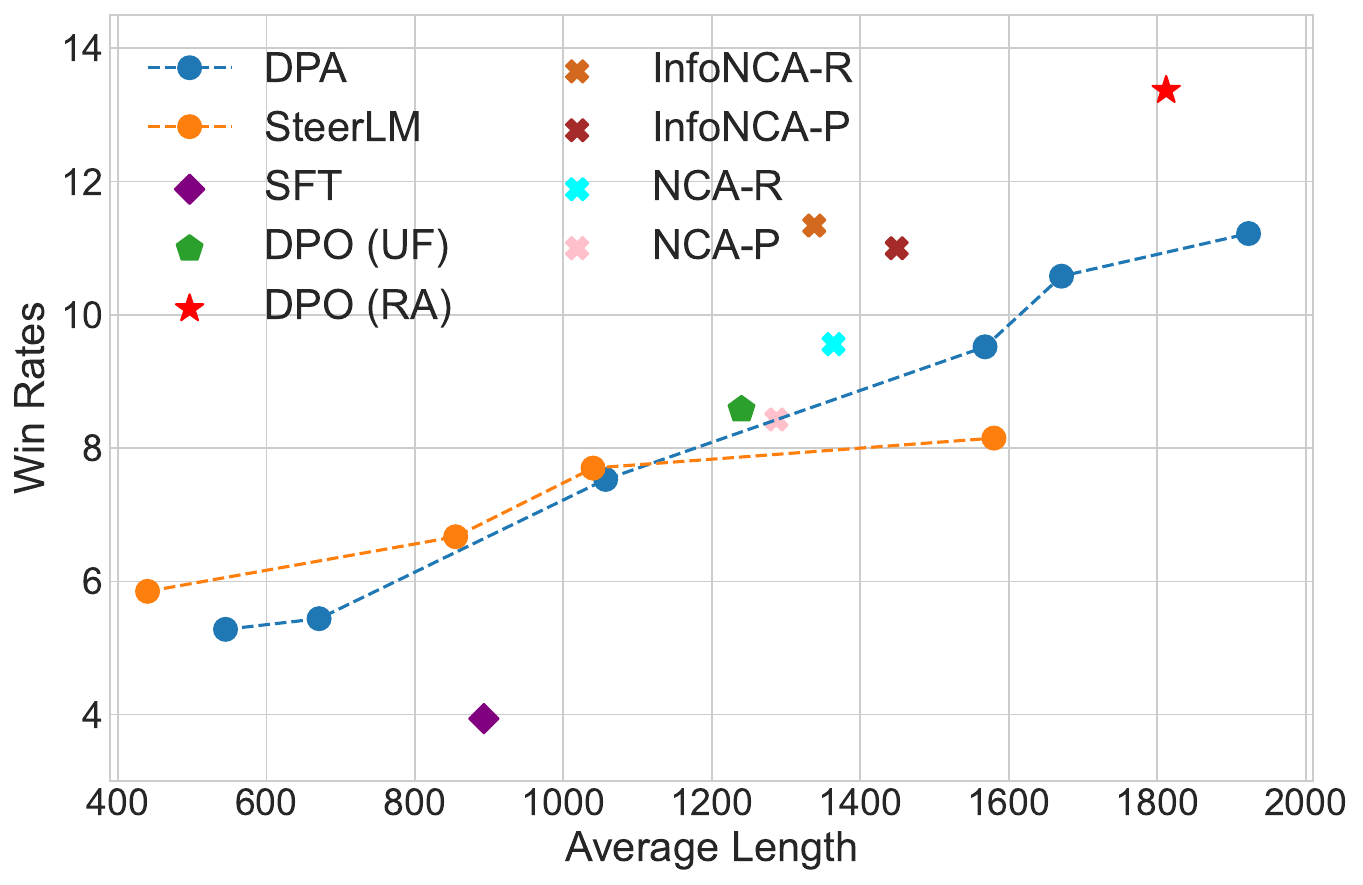}
    \vspace{-0.2cm}
        \caption{Comparisons with DPA, SteerLM, and (Info)NCA.\vspace{-0.3cm}}
    \label{fig_baselines}
\end{figure}

\vspace{-0.1cm}
\paragraph{Controllable Generation with Prompt.}
In Table \ref{tab_system}, we ablate how generations differ when changing the goal rewards in the system prompt. We observe that the AlpacaEval scores of the Qwen2-7B-It+DPO (RA) model change accordingly as $g$ varies. However, using the same $g=10$ prompt during inference for Qwen2-7B-It+DPO (UF) fails to give competitive results, indicating that our method is superior not only because of the additional system prompt.
\begin{table}[htbp]
\centering
\begin{tabular}{l|ccc|c}
\hline
&\begin{tabular}{@{}c@{}}$g=10$\end{tabular} & \begin{tabular}{@{}c@{}}$g=8$ \end{tabular}&\begin{tabular}{@{}c@{}}$g=6$\end{tabular} & \begin{tabular}{@{}c@{}}UF ($g=10$) \end{tabular} \\ \hline
LC WR & \textbf{31.17} & 28.66 & 25.56 & 24.44\\ 
WR & \textbf{27.58} & 25.57 & 18.88 & 20.75 \\ \hline
\end{tabular}
\caption{Performance when conditioned on different goal rewards in the inference prompt.\vspace{-0.2cm}}
\label{tab_system}
\end{table}

We refer the readers to Appendix \ref{app_more_abl} for more ablations.
\section{Conclusion}
In this paper, we first investigate the limitations of direct alignment algorithms, which arise from focusing solely on relative preferences while neglecting the qualities of the responses and their gaps. Specifically, since many rejected responses are only slightly worse than the chosen ones, striving to maximize the reparameterized reward gap will cause overfitting and unnecessarily unlearning the high-quality rejected response. Moreover, the directly aligned LLMs often struggle to differentiate between responses of varying quality, indiscriminately learning the low-quality chosen responses and failing to generalize effectively to more optimal responses that are sparse in the preference data. To resolve the above limitations, we introduce a straightforward solution---learning reward-conditioned policies. By optimizing the LLM to generate responses conditioned on their qualities, it can better differentiate between quality levels and learn from the entire spectrum. Motivated by this, we propose a data relabeling method that constructs reward-augmented datasets by conditioning on the quality of responses as the goal quality. In experiments, we fine-tune various LLMs by applying DPO on our reward-augmented data. The results demonstrate that our approach consistently delivers significant performance improvements across various instruction-following benchmarks and increases the average accuracy on academic benchmarks. 

\section*{Impact Statement}
This paper presents work whose goal is to advance the field of Machine Learning. There are many potential societal consequences of our work, none of which we feel must be specifically highlighted here.

\bibliography{reference}
\bibliographystyle{icml2025}

\appendix
\onecolumn
\section{Theory}
In this section, we present the theoretical analysis for our proposed method.
\subsection{Concepts}
We provide some useful concepts for the simplicity of later discussions. 
\begin{itemize}
    \item Hellinger distance $D_{\mathrm{Hellinger}}(p\|q)$ between two probability density functions $p$ and $q$ defined on $\mathcal
X$ is defined as 
\begin{align*} D_{\mathrm{Hellinger}}(p\|q) = \frac{1}{2}{\int_{x\in\mathcal{X}}\left(\sqrt{p(x)} - \sqrt{q(x)}\right)^2}{}\mathrm{d}x.
\end{align*}
   \item Total variation (TV) distance $D_{\mathrm{TV}}(p\|q)$ between two probability density functions $p$ and $q$ defined on $\mathcal
X$ is defined as 
\begin{align*} D_{\mathrm{TV}}(p\|q) = \frac{1}{2}{\int_{x\in\mathcal{X}}|p(x) - q(x)}|\mathrm{d}x.
\end{align*}
 \item Kullback–Leibler (KL) divergence $\text{KL}(p\|q)$ between two probability density functions $p$ and $q$ defined on $\mathcal
X$ is defined as 
\begin{align*} \text{KL}(p\|q) = \int_{x\in\mathcal{X}}\log\left(\frac{ p(x)}{q(x)}\right)p(x)\mathrm{d}x.
\end{align*}
\item We denote $\mathcal{N}_\epsilon(\mathcal{F},\|\cdot\|_\infty)$ as the $\epsilon$-covering number \citep{zhou2002covering} for function class $\mathcal{F}$ under the infinity norm $\|\cdot\|_\infty$. Widely used in theoretical analysis \citep{liu2024provably,zhan2023provable}, the $\epsilon$-covering number characterizes the complexity of the function class $\mathcal{F}$. 
\end{itemize}

\subsection{Theoretical Formulation}
\noindent
\textbf{Goal-conditioned preference model.} Consider a language model $\pi\in\Delta_\mathcal{Y}^\mathcal{X}$ that takes the prompt $x\in\mathcal{X}$ as input and outputs the response $y\in\mathcal{Y}$, where $\mathcal{X}$ and $\mathcal{Y}$ are spaces of prompts and responses, respectively.
Given the prompt $x\in\mathcal{X}$, a discrete probability distribution $\pi(\cdot\mid x)\in\Delta_\mathcal{Y}$ is generated, where $\Delta_\mathcal{Y}$ is the set of discrete distributions over $\mathcal{Y}$. 
We define the goal-conditioned reward function class as $\mathcal{R} \subset\{R(x,y,g): \mathcal{X}\times\mathcal{Y}\times\mathcal{G}\mapsto \mathbb{R}\}$, where $\mathcal{G}$ is the goal space. 
The goal-conditioned Bradley-Terry model \citep{bradley1952rank} for annotations is described as
\begin{align}
    \mathbb{P}_R(y_1 \succ y_0|x,y_1, y_0,g) = \frac{\exp(R(x,y_1,g))}{\exp(R(x,y_1,g))+ \exp(R(x,y_0, g))}= \sigma\big(R(x,y_1,g) -R(x,y_0,g)\big),\label{eq: bt model}
\end{align}
where $\sigma(z) = 1/(1+\exp(-z))$ is the sigmoid function. For notational simplicity, we also denote that the reward is parameterized by $\theta\in\Theta$. 
We denote the corresponding negative log-likelihood function for $r$ for a reward-augmented preference dataset $\bar{\mathcal{D}}=\{(x^i,y^i_w,y^i_l,g^i)\}_{i=1}^N$ as
\begin{align}
     \mathcal{L}(R,\bar{\mathcal{D}} )= -\EE_{(x, y_w, y_l,g)\sim\bar{\mathcal{D}}}\bigl[\log\sigma\bigl(R(x, y_w, g) -R(x, y_l, g)\bigr)\bigr],\label{eq:mle-loss}
\end{align}
 where $y_w^i$ is preferred to $y_l^i$ by the annotation given the prompt $x^i$ and goal $g^i$ for any $i\in[N]$.
For notational simplicity, we denote the DPO loss by 
\begin{align}
    \mathcal{L}_{\text{DPO}}(\pi,\bar{\mathcal{D}} )= -\EE_{(x, y_w, y_l,g)\sim\bar{\mathcal{D}}}\biggl[\log\sigma\biggl(\beta\log\frac{\pi(y_w\mid x, g)}{\pi_{\text{ref}}(y_w\mid x,g)} - \beta\log\frac{\pi(y_l\mid x,g)}{\pi_{\text{ref}}(y_l\mid x, g)}\biggr)\biggr].\label{eq:dpo_loss}
\end{align}

\noindent
\textbf{Performance metric.} For notational simplicity in the theoretical analysis, we denote by $R^\star$ the ground-truth goal-conditioned reward function.
The alignment target is to maximize the expected true reward $R^{\star}$ conditioned on the optimal goal $g^\star\in\mathcal{G}$.
Thus, we define the value function of any policy $\pi$ as 
\begin{align}
    J(\pi) = \mathbb{E}_{x\sim d_0,y\sim \pi(\cdot|x,g^\star)}\big[R^{\star}(x,y,g^\star)\big].\label{eq: value function}
\end{align}
Here we allow the prompt distribution $d_0(\cdot)$ to be different from that of the offline dataset distribution $\mu_{\bar{\cD}}(\cdot)$, but it is assumed to be known.
In the meanwhile, we consider the policies that share the same support as the reference policy $\pi_{\mathrm{ref}}$ \citep{xiong2023gibbs}, that is, we take a policy class $\Pi$ as
\begin{align}
    \Pi = \Big\{\pi:\cX\times\mathcal{G}\mapsto\Delta(\cA)\,\Big|\, \mathrm{Supp}(\pi(\cdot|x,g))\subseteq\mathrm{Supp}(\pi_{\mathrm{ref}}(\cdot|x,g)),\,\,\forall (x,g)\in\cX\times\mathcal{G}\Big\}.\label{eq: policy class}
\end{align}

The performance gap of a learned policy $\widehat{\pi}\in\Pi$ w.r.t. any given optimal policy $\pi^\star$ is measured as 
\begin{align}
    \mathrm{Gap}^{\pi^\star}(\widehat{\pi}) = J(\pi^*) - J(\widehat{\pi}),\,\,\emph{\text{given any optimal policy $\pi^\star\in\Pi$,}}\label{eq: gap}
\end{align}
One popular choice to define the optimal policy is to maximize the KL-regularized reward, i.e., 
\begin{align}
    \pi^\star = \operatornamewithlimits{argmax}_{\pi\in\Pi}\Bigl[R^\star(x,y,g^\star)-\beta_0\text{KL}\bigl(\pi(\cdot\mid x,g^\star)\| \pi_{\text{ref}}(\cdot\mid x,g^\star)\bigr)\Bigr]
\end{align}
for a fixed $\beta_0>0$.

\noindent
\textbf{Theoretical version of the reward-augmented DPO.} We formulate the theoretical version of the reward-augmented DPO in Algorithm \ref{alg:theory}, where we add an SFT regularizer on the empirical objective to handle the issue of distribution shift and analyze the bound on the suboptimality \citep{liu2024provably,cen2024value}. One simple choice  to define SFT policy $\pi_{\text{SFT}}$ is to utilize the chosen labels in the original preference dataset $\mathcal{D}$, that is,
    \begin{align}
        \pi_{\text{sft}} = \operatornamewithlimits{argmax}_{\pi\in\Pi} \mathbb{E}_{(x,y_w)\sim\mathcal{D}}[\log \pi(y_w\mid x, g^\star)].
    \end{align}
In practice, the goal relabeling distribution $g\sim p_\mathcal{G}(\cdot\mid x,y)$ is set to be a deterministic selection of the annotated reward of the chosen response, i.e., $g= r_{\text{RM}}(x, y)$ for any $i\in[N]$ and a given reward model $r_{\text{RM}}$. We also remark that the size of the reward-augmented preference dataset $\bar{\mathcal{D}}$ is $N=2N_0$, where $N_0$ denotes the size of the original preference dataset $\mathcal{D}$.

\subsection{Assumptions for Theoretical Analysis}
Similar to the theoretical analyses on offline RLHF \citep{liu2024provably, cen2024value}, we provide the following assumptions. 

\begin{algorithm}[t]
	\caption{Theoretical Version of the Reward-Augmented DPO}
	
	\begin{algorithmic}[1] \label{alg:1} 
	\STATE \textbf{Input}: Preference dataset $\mathcal{D}=\{(x^i,y^i_w,y^i_l)\}_{i=1}^{N_0}$, parameters $\beta,\eta>0$, reference policy $\pi_{\mathrm{ref}}$, SFT policy $\pi_{\text{sft}}$ for the regularizer, and goal labeling distribution $p_\mathcal{G}$.
    \STATE Initialize the reward-augmented preference dataset $\bar{\mathcal{D}}=\emptyset$.
    \FOR{$i=1,\ldots, N_0$}
 \STATE Sample goal $g^i_w$ from $p_\mathcal{G}(\cdot\mid x^i, y^i_w)$ and update  the reward-augmented preference dataset as $\bar{\mathcal{D}}\leftarrow\bar{\mathcal{D}}\cup\{(x^i,y_w^i,y_l^i,g^i_w)\}$.
\STATE Sample goal $g^i_l$ from $p_\mathcal{G}(\cdot\mid x^i, y^i_l)$ and update  the reward-augmented preference dataset as $\bar{\mathcal{D}}\leftarrow\bar{\mathcal{D}}\cup\{(x^i,y_l^i,y_w^i,g^i_l)\}$.
     \ENDFOR
    \STATE   Solve policy $\pi_{\hat\theta}$ by optimizing the following objective
    \begin{align}
\min_{\theta\in\Theta}\biggl\{{\mathbb{E}_{\substack{x\sim d_0, y_0\sim \pi_{\text{sft}}(\cdot\mid x,g^\star) }}\Big[ -\eta\beta\cdot\log(\pi_\theta(y_0|x,g^\star))\Big]} +    \mathcal{L}_{\text{DPO}}(\pi_\theta,\bar{\mathcal{D}})\bigg\}
    \end{align}
    \STATE\textbf{Output}: Policy $\hat{\pi} = \pi_{\hat\theta}$.
    \label{alg:theory}
	\end{algorithmic}
\end{algorithm}

\begin{assumption}[True reward model]\label{as: regularity}
    We assume that the true goal-conditioned reward model $R^{\star}\in\cR$ for, and for any $R\in\cR$ and $(x,y,g)\in\cX\times\cA\times\mathcal{G}$, it holds that $R(x,y,g)\in[-B/2,B/2]$ for a positive constant $B>0$.
\end{assumption}
Assumption~\ref{as: regularity} is standard in sample complexity analysis \citep{zhu2023principled, zhan2023provable, ye2024theoretical} in RLHF.
\begin{assumption}
   [Regularity]\label{as: regularity2} 
    We assume that the reward model class $\cR$, prompt space $\mathcal{X}$, and goal space $\mathcal{G}$ are convex and compact.  
\label{as:exchange}
\end{assumption}
Assumption~\ref{as: regularity2} plays a role in establishing the equivalence between maximin and minimax optimizations. This assumption is naturally satisfied when considering a linear reward function \citep{zhu2023principled, xiong2023gibbs, cen2024value} of the form \( R_\theta(x, y, g) = \varphi(x, y, g)^\top \theta \), where \(\varphi\) represents a known feature map. More broadly, the assumption is also met by the class of Lipschitz continuous reward models.
\begin{assumption}[Partial coverage coefficient]\label{as: coverage}
    Given the optimal policy $\pi^\star\in\Pi$, the coverage coefficient of the population  distribution $\mu_{\bar{\cD}}$ of the reward-augmented preference dataset $\bar{\cD}$  w.r.t. reward model class $\cR$, optimal policy $\pi^\star$, and the SFT policy $\pi_{\mathrm{sft}}$, denoted by $C_{\mu_{\bar{\cD}}}(\cR; \pi^\star, \pi_{\mathrm{sft}})$, is defined as 
    \begin{align}
\sup_{R\in\cR}\frac{\mathbb{E}_{x\sim d_0,y_1\sim \pi^\star(\cdot|x, g^\star),y_0\sim \pi_\text{sft}(\cdot|x,{g^\star})}\big[(R^{\star}(x,y_1,g^\star) -R^{\star}(x,y_0,g^\star) - (R(x,y_1,g^\star) -R(x,y_0,g^\star))\big]}{\sqrt{\mathbb{E}_{(x,y_w,y_l,g)\sim \mu_{\bar{\cD}}}\left[\big|(R(x,y_w,g) -R(x,y_l,g)) - (R(x,y_w,g) -R(x,y_l,g))\big|^2\right]}}.
    \end{align}
    We assume that $C_{\mu_{\bar{\cD}}}(\cR; \pi^\star,\pi_{\text{sft}})<+\infty$ for the given optimal policy $\pi^\star\in\Pi$. 
    
\end{assumption}

Assumption~\ref{as: coverage} characterizes how well the dataset $\bar{\cD}$ covers the optimal policy $\pi^\star$ and the SFT policy $\pi_{\text{sft}}$ given the optimal goal $g^\star$, instead of covering all the policies in the policy class. That is the reason why we call this assumption ``partial coverage''. Different variants of partial coverage assumptions are posed in previous literature \citep{liu2024provably,cen2024value,zhan2023provable,xie2021bellman} that study offline RLHF and RL to characterize the distribution shift between the optimal policy and the offline dataset distribution. 
We remark that the quantity $C_{\mu_{\bar{\cD}}}(\cR; \pi^\star,\pi_{\text{sft}})$ is upper bounded by the density ratio $\|d_0(\cdot)\otimes \pi^\star(\cdot| \cdot,g^\star)\otimes \pi_{\mathrm{sft}}(\cdot| \cdot,g^\star)/\mu_{\bar{\cD}}(\cdot,\cdot,\cdot,g^\star)\|_{\infty}$.

\subsection{Theoretical Results}
\label{appendix_thm}
Under assumptions introduced before, we are ready to give the theoretical result for Algorithm \ref{alg:1} in the following theorem.
\begin{theorem}[Suboptimality of Algorithm~\ref{alg:theory}]\label{thm:main}
    Taking the policy class $\Pi$ as \eqref{eq: policy class}, supposing that Assumptions \ref{as: regularity}, \ref{as:exchange}, and \ref{as: coverage} hold, and assuming that the reward model class $\cR$ has a finite $\varepsilon$-covering number under $\|\cdot\|_{\infty}$-norm $\cN_{\varepsilon}(\cR,\|\cdot\|_{\infty})<+\infty$ with $\varepsilon = (6\cdot(1+e^B)\cdot N)^{-1}$.
    Setting 
    \begin{align*}
       \eta =  (1+\exp(B))^{-2}\cdot\sqrt{24\log\left(\cN_{\varepsilon}(\cR,\|\cdot\|_{\infty})/\delta\right)/N},\quad \beta = 1/\sqrt{N}
    \end{align*} 
    in Algorithm~\ref{alg:1}. 
    Then the output policy $\widehat \pi$ of Algorithm~\ref{alg:1} satisfies that with probability at least $1-\delta$,
    \begin{align}
        &\mathrm{Gap}^{\pi^\star}(\widehat{\pi})\!\leq\! \sqrt{\frac{1}{N}}\cdot\bigg\{\frac{\sqrt{6}}{4} \big(1+\exp(B)\big)^2 \big(\big(C_{\mu_{\bar{\cD}}}(\cR; \pi^\star, \pi_\text{sft})\big)^2+1\big) \iota \notag\\&\qquad\qquad\qquad+  \mathbb{E}_{x\sim d_0}\Big[\mathrm{KL}\big(\pi^\star(\cdot|x, g^\star)\|\pi_{\mathrm{ref}}(\cdot|x, g^\star)\big) \Big]\bigg\},
    \end{align}
    where $\iota =  \sqrt{\log\left(\cN_{\varepsilon}(\cR,\|\cdot\|_{\infty})/\delta\right)}$ with $\varepsilon = (6\cdot(1+e^B)\cdot N)^{-1}$.
    Here, $N$ denotes the number of preference pairs in $\cD$, $R$ denotes the upper bound of the reward models, and the partial coverage coefficient $C_{\mu_{\bar{\cD}}}(\cR; \pi^\star,\pi_\text{sft})$ is defined in Assumption~\ref{as: coverage}.
\end{theorem}
The detailed proof is provided in Appendix \ref{app:proof}. Theorem \ref{thm:main} shows that our proposed reward-augmented DPO (Algorithm \ref{alg:1}) can attain global convergence to the optimal policy and the suboptimality decays at the order of $N^{-1/2}$ ($N$ denotes the size of the reward-augmented preference dataset). Theorem \ref{thm:main} provides a theoretical justification for the strong empirical performance of the reward-augmented DPO introduced in this paper. Unlike prior works on goal-conditioned reinforcement learning with supervised learning \citep{yang2022rethinking,ghosh2019learning}, which typically establish weaker results such as local performance improvements or the optimization of a lower bound on $J(\pi)$, our analysis guarantees global convergence to the optimal policy. This distinction underscores the significance of integrating DPO-like methods with goal-conditioned approaches.

\subsection{Proof of Theorem \ref{thm:main}}\label{app:proof}

\paragraph{Bridge Algorithm \ref{alg:1} to the maximin optimization.} Motivated by \citet{liu2024provably}, we transform the optimization objective in Algorithm \ref{alg:1} to a minimax optimization objective, and then to a maximum optimization objective, where the maximum optimization objective can be analyzed with tools in RL analysis.

Define the objective function $\phi(\pi, R)$  as
\begin{align}
    \phi(\pi, R)&= \eta\cdot\mathbb{E}_{\substack{x\sim d_0, y_1\sim\pi(\cdot|x,g^\star)\notag\\ y_0\sim \pi_\text{sft}(\cdot|x,g^\star)}}\Big[R(x,y_1,g^\star) -R(x,y_0,g^\star) \notag\\ &\qquad\qquad-\beta\cdot D_{\mathrm{KL}}\big(\pi(\cdot|x,g^\star)\|\pi_{\mathrm{ref}}(\cdot|x,g^\star)\big)\Big] +  {\cL}(R,\bar{\mathcal{D}}).\label{eq:phi}
\end{align}
First, we prove that the derived policy $\hat \pi$ from Algorithm \ref{alg:1} satisfies
\begin{align}   \hat \pi\in\argmax_{\pi\in\Pi}\phi(\widehat{R},\pi),\quad \text{where}\quad  \widehat{R} \in \argmin_{R\in \cR}\,\max_{\pi\in\Pi}\phi(\pi, R).\label{eq: equivalent formal 1}
 \end{align}
By the definition of the optimization objective $\phi(\pi,R)$ in \eqref{eq:phi}, we have
\begin{align}  \min_{R\in \cR}\,\max_{\pi\in\Pi}\phi(\pi, R) &= \min_{R\in \cR}\Bigg\{ \eta\cdot\max_{\pi\in\Pi}\bigg\{\mathbb{E}_{x\sim d_0, y_1\sim\pi(\cdot|x,g^\star)}\Big[R(x,y_1,g^\star) -\beta\cdot \mathrm{KL}\big(\pi(\cdot|x, g^\star)\|\pi_{\mathrm{ref}}(\cdot|x,g^\star)\big)\Big]\bigg\} \notag\\
&\qquad\qquad - \eta\cdot\mathbb{E}_{x\sim d_0, y_0\sim\pi_{\text{sft}}(\cdot|x,g^\star)}\Big[R(x,y_0,g^\star)\Big] +  {\cL}(R,\bar{\mathcal{D}})\Bigg\}.\label{eq: equivalent formal 1.1}    \end{align}
Then, we apply the following lemma to solve the inner maximization problem in \eqref{eq: equivalent formal 1.1}.
\begin{lemma}[Oracle optimal KL-regularized policy]\label{lem: optimal policy}
    Given any reward model $R\in\mathcal{R}$, the optimal policy $\pi_R$ to the maximization problem \begin{align}
            \max_{\pi\in\Pi}\bigg\{\mathbb{E}_{x\sim d_0, y\sim\pi(\cdot|x,g^\star)}\Big[R(x,y,g^\star) -\beta\cdot \mathrm{KL}\big(\pi(\cdot|x, g^\star)\|\pi_{\mathrm{ref}}(\cdot|x,g^\star)\big)\Big]\bigg\}.\label{eq: kl-regularized reward}
    \end{align} is given by 
\begin{align}
    \pi_R(\cdot|x, g)&= \frac{1}{Z_R(x, g)}\cdot \pi^{\mathrm{ref}}(\cdot|x,g)\cdot \exp\left(\beta^{-1}R(x,\cdot,g)\right),\\\notag
    Z_R(x, g) &= \int_{y\in\cY} \exp\left(\beta^{-1}R(x,y, g)\right)\mathrm{d}\pi_{\mathrm{ref}}(y|x, g),
\end{align}
and correspondingly the optimal value of \eqref{eq: kl-regularized reward} is given by $\eqref{eq: kl-regularized reward} =  \mathbb{E}_{x\sim d_0}[\beta\cdot \log(Z_R(x, g^\star))]$.
\end{lemma}
\begin{proof}[Proof of Lemma \ref{lem: optimal policy}]
    See the proof in Lemma 4.2 of \citet{liu2024provably}.
\end{proof}
By Lemma \ref{lem: optimal policy} and \eqref{eq: equivalent formal 1.1}, we have
\begin{align}
     \min_{R\in \cR}\,\max_{\pi\in\Pi}\phi(\pi, R) &= \min_{R\in \cR}\left\{ \beta\eta\cdot\log (Z_R(x,g^\star))- \eta\cdot\mathbb{E}_{x\sim d_0, y_0\sim\pi_{\text{sft}}(\cdot|x,g^\star)}\Big[R(x,y_0,g^\star)\Big] +  {\cL}(R,\bar{\mathcal{D}})\right\}.\label{eq:minimax-3}
\end{align}
From Lemma~\ref{lem: optimal policy}, we know that given any reward model $R\in\mathcal{R}$, we can reparameterize it via its corresponding optimal goal-conditioned KL-regularized policy $\pi_R$ \citep{rafailov2024direct}, that is,
\begin{align}
   R(x,\cdot,g) = \beta\cdot\log\left(\frac{\pi_R(\cdot|x,g)}{\pi^{\mathrm{ref}}(\cdot|x,g)}\right) + \beta\cdot\log (Z_R(x,g)).\label{eq: reparametrize}
\end{align}
Plugging \eqref{eq: reparametrize} into \eqref{eq:minimax-3}, we show that the optimization problem in Algorithm \ref{alg:1} relates to the minimax optimization problem on $\phi(\pi, R)$:
\begin{align}
     \min_{R\in \cR}\,\max_{\pi\in\Pi}\phi(\pi, R) &= \min_{R\in \cR}\left\{ \eta\beta\cdot \mathbb{E}_{x\sim d_0, y_0\sim\pi_{\text{sft}}(\cdot|x,g^\star)}\Big[\log\left(\frac{\pi_R(y_0\mid x,g^\star)}{ \pi_{\text{ref}}(y_0\mid x,g^\star)}\right)\Big] +  {\cL}_{\text{DPO}}(\pi_R,\bar{\mathcal{D}})\right\}\notag\\
     &= \min_{R\in \cR}\left\{ \eta\beta\cdot \mathbb{E}_{x\sim d_0, y_0\sim\pi_{\text{sft}}(\cdot|x,g^\star)}\Big[\log\left({\pi_R(y_0\mid x,g^\star)}\right)\Big] +  {\cL}_{\text{DPO}}(\pi_R,\bar{\mathcal{D}})\right\}.\label{eq:minimax-3}
\end{align}
where the first equality uses the definition of DPO loss $\mathcal{L}_{\text{DPO}}$ in \eqref{eq:dpo_loss}. Since we know that $\hat \pi\in\argmax_{\pi\in\Pi}\phi(\widehat{R},\pi)$ and $\hat r$ solves the minimization problem in \eqref{eq:minimax-3}, we know that $\hat \pi = {\pi}_{\widehat{R}}$ by Lemma \ref{lem: optimal policy}. 

Next, we show that the minimization problem $\phi(\pi, R)$ can be equivalently transformed into a maximization problem. Specifically, we will prove that the output policy $\hat \pi$ for the Algorithm \ref{alg:1} satisfies
\begin{align}   \hat \pi\in\argmax_{\pi\in\Pi} \min_{R\in\mathcal{R}}\phi(\pi, R),\label{claim:maxmin}
 \end{align}
which is implied by the following theorem.

\begin{theorem}
     For the policy class $\Pi$ defined in \eqref{eq: policy class} and the reward model class $\cR$ satisfying Assumption~\ref{as: regularity2}, consider the following policy defined as
    \begin{align}   \pi_{\widehat{R}}\in\argmax_{\pi\in\Pi}\phi(\widehat{R},\pi),\quad \text{where}\quad  \widehat{R} \in \argmin_{R\in \cR}\,\max_{\pi\in\Pi}\phi(\pi, R).\label{eq: equivalent formal 1}
    \end{align}
    Then the policy $\pi_{\widehat{R}}$ also solves the following maximin optimization: 
    \begin{align}
        \pi_{\widehat{R}} \in \argmax_{\pi\in\Pi}\,\min_{R\in\cR} \phi(\pi, R).
    \end{align} \label{thm:minimax}
\end{theorem}
\begin{proof}
    Under Assumption \ref{as: regularity}, we know that $\phi(\pi, R)$ is convex for $R\in\mathcal{R}$ and strongly concave for $\pi\in{\Pi}$. Applying Theorem 5.6 in \citet{liu2024provably}, we prove Theorem \ref{thm:minimax}.
\end{proof}

\paragraph{Suboptimality Decomposition.}
By the definitions of the optimization objective $\phi(\pi, R)$ in \eqref{eq:phi} and the suboptimality gap of $\widehat{\pi}$ w.r.t. $\pi^\star$ in \eqref{eq: gap}, we decompose the gap as 
    \begin{align}
        &\mathrm{Gap}^{\pi^\star}(\widehat{\pi})\notag\\
        &\qquad=\mathbb{E}_{x\sim d_0,y\sim \pi^\star(\cdot|x,g^\star)}\big[R^{\star}(x,y, g^\star)\big] - \mathbb{E}_{x\sim d_0,a\sim \widehat{\pi}(\cdot|x, g^\star)}\big[R^{\star}(x,y, g^\star)\big]\notag\\
        &\qquad= \mathbb{E}_{x\sim d_0,y_1\sim \pi^\star(\cdot|x, g^\star), y_0\sim \pi_{\mathrm{sft}}(\cdot|x, g^\star)}\Big[R^{\star}(x,y_1, g^\star) -R^{\star}(x,y_0, g^\star) - \beta\cdot \mathrm{KL}\big(\pi^\star(\cdot|x, g^\star)\|\pi_{\mathrm{ref}}(\cdot|x, g^\star)\big)\Big]\notag\\
        &\qquad\qquad - \eta^{-1}\cdot\min_{R\in\mathcal{R}}\phi(\hat \pi, R)+\eta^{-1}\cdot\min_{R\in\mathcal{R}}\phi(\hat \pi, R)\notag\\
        &\qquad\qquad - \mathbb{E}_{x\sim d_0,y_1\sim \widehat{\pi}(\cdot|x, g^\star), y_0\sim \pi_{\text{sft}}(\cdot|x, g^\star)}\Big[R^{\star}(x,y_1, g^\star) -R^{\star}(x,y_0, g^\star)- \beta\cdot \mathrm{KL}\big(\widehat{\pi}(\cdot|x, g^\star)\|\pi_{\mathrm{ref}}(\cdot|x, g^\star)\big)\Big]\notag\\
        &\qquad\qquad + \beta \cdot \mathbb{E}_{x\sim d_0}\Big[\mathrm{KL}\big(\pi^\star(\cdot|x, g^\star)\|\pi_{\mathrm{ref}}(\cdot|x, g^\star)\big) - \mathrm{KL}\big(\widehat{\pi}(\cdot|x, g^\star)\|\pi_{\mathrm{ref}}(\cdot|x, g^\star)\big) \Big]\notag\\
        &\qquad:=\text{Term (A)} + \text{Term (B)} + \text{Term (C)} \label{eq: gap proof},
    \end{align}
    where we abbreviate Term (A), Term (B), and Term (C) as follows
    \begin{align}
        &\text{Term (A)}=-\eta^{-1}\cdot\min_{R\in\mathcal{R}}\phi(\hat \pi, R)\notag\\
        &\qquad = \mathbb{E}_{x\sim d_0,y_1\sim \pi^\star(\cdot|x, g^\star),y_0\sim \pi_{\text{sft}}(\cdot|x, g^\star)}\Big[R^{\star}(x,y_1, g^\star) -R^{\star}(x,y_0, g^\star) - \beta\cdot \mathrm{KL}\big(\pi^\star(\cdot|x, g^\star)\|\pi_{\mathrm{ref}}(\cdot|x, g^\star)\big)\Big],
    \end{align}
    \begin{align}
        &\text{Term (B)} = \eta^{-1}\cdot\min_{R\in\mathcal{R}}\phi(\hat \pi, R)\notag\\
        &\qquad\qquad   - \mathbb{E}_{x\sim d_0,y_1\sim \widehat{\pi}(\cdot|x, g^\star),y_0\sim \pi_{\mathrm{sft}(\cdot|x, g^\star)}}\Big[R^{\star}(x,y_1, g^\star) -R^{\star}(x,y_0, g^\star) - \beta\cdot \mathrm{KL}\big(\widehat{\pi}(\cdot|x, g^\star)\|\pi_{\mathrm{ref}}(\cdot|x, g^\star)\big)\Big],
    \end{align}
    and
    \begin{align}
        \text{Term (C)} = \beta \cdot \mathbb{E}_{x\sim d_0}\Big[\mathrm{KL}\big(\pi^\star(\cdot|x, g^\star)\|\pi_{\mathrm{ref}}(\cdot|x, g^\star)\big) - \mathrm{KL}\big(\widehat{\pi}(\cdot|x, g^\star)\|\pi_{\mathrm{ref}}(\cdot|x, g^\star)\big) \Big].
    \end{align}
    In the following, we bound Term (A) and Term (B) respectively.

    \paragraph{Analysis of Term (A) in \eqref{eq: gap proof}.} 
    Note that
    \begin{align}
        &\text{Term (A)}\notag\\
        &\qquad = \mathbb{E}_{x\sim d_0,y_1\sim \pi^\star(\cdot|x, g^\star),y_0\sim \pi_{\text{sft}}(\cdot|x, g^\star)}\Big[R^{\star}(x,y_1, g^\star) -R^{\star}(x,y_0, g^\star) - \beta\cdot \mathrm{KL}\big(\pi^\star(\cdot|x, g^\star)\|\pi_{\mathrm{ref}}(\cdot|x, g^\star)\big)\Big] \label{eq: term a}\notag\\
        &\qquad\qquad  -\eta^{-1}\cdot\min_{R\in\mathcal{R}}\phi(\hat \pi, R)\notag\\
        &\qquad\leq  \mathbb{E}_{x\sim d_0,y_1\sim \pi^\star(\cdot|x, g^\star),y_0\sim \pi_{\mathrm{sft}}(\cdot|x, g^\star)}\Big[R^{\star}(x,y_1, g^\star) -R^{\star}(x,y_0, g^\star) - \beta\cdot \mathrm{KL}\big(\pi^\star(\cdot|x, g^\star)\|\pi_{\mathrm{ref}}(\cdot|x, g^\star)\big)\Big]\notag\\
        &\qquad\qquad  -\eta^{-1}\cdot \min_{R\in\mathcal{R}}\phi( \pi^\star, R)\notag\\
        &\qquad= \max_{R\in\mathcal{R}}\Bigg\{\mathbb{E}_{x\sim d_0,y_1\sim \pi^\star(\cdot|x, g^\star),y_0\sim \pi_{\text{sft}}(\cdot|x, g^\star)}\Big[\big(R^{\star}(x,y_1, g^\star) -R^{\star}(x,y_0, g^\star)\big) - \big(R(x,y_1, g^\star) -R(x,y_0, g^\star)\big)\Big] \notag\\
        &\qquad\qquad\qquad- \eta^{-1}\cdot \cL(R,\bar{{\cD}})\Bigg\},  
    \end{align}
    where the inequality follows the fact that $\hat \pi$ solves the maxmin optimization problem in \eqref{claim:maxmin}.

    \paragraph{Analysis of Term (B) in \eqref{eq: gap proof}.} Note that
    \begin{align}
        &\text{Term (B)}\notag\\
        &\qquad = \eta^{-1}\cdot\min_{R\in\mathcal{R}}\phi(\hat \pi, R)\notag\\
        &\qquad \qquad  - \mathbb{E}_{x\sim d_0,y_1\sim \widehat{\pi}(\cdot|x, g^\star),y_0\sim \pi_{\mathrm{sft}(\cdot|x, g^\star)}}\Big[R^{\star}(x,y_1, g^\star) -R^{\star}(x,y_0, g^\star) - \beta\cdot \mathrm{KL}\big(\widehat{\pi}(\cdot|x, g^\star)\|\pi_{\mathrm{ref}}(\cdot|x, g^\star)\big)\Big]\notag\\
        &\qquad  \leq \mathbb{E}_{x\sim d_0, y_1\sim\widehat{\pi}(\cdot|x, g^\star),y_0\sim \pi_{\text{sft}}(\cdot|x, g^\star)}\Big[R^{\star}(x,y_1, g^\star) -R^{\star}(x,y_0, g^\star) -\beta\cdot \mathrm{KL}\big(\widehat{\pi}(\cdot|x, g^\star)\|\pi_{\mathrm{ref}}(\cdot|x, g^\star)\big)\Big] \notag\\
        &\qquad\qquad+ \eta^{-1}\cdot \cL(R^{\star},\bar{{\cD}})\notag\\
        &\qquad \qquad  - \mathbb{E}_{x\sim d_0,y_1\sim \widehat{\pi}(\cdot|x, g^\star),y_0\sim \pi_{\mathrm{sft}(\cdot|x, g^\star)}}\Big[R^{\star}(x,y_1, g^\star) -R^{\star}(x,y_0, g^\star) - \beta\cdot \mathrm{KL}\big(\widehat{\pi}(\cdot|x, g^\star)\|\pi_{\mathrm{ref}}(\cdot|x, g^\star)\big)\Big]\notag\\
        &\qquad= \eta^{-1}\cdot \cL(R^{\star},\bar{{\cD}}),\label{eq: term b}
    \end{align}
    where the inequality uses the fact that $R^{\star}\in\cR$ by Assumption~\ref{as: regularity} and the definition of the optimization objective in \eqref{eq:phi}.

    \paragraph{Concluding the remaining proof.}
    Combining \eqref{eq: gap proof}, \eqref{eq: term a}, and \eqref{eq: term b}, we have
    \allowdisplaybreaks
    \begin{align}
    &\mathrm{Gap}_{\beta}^{\pi^\star}(\widehat{\pi}) = \text{Term (A)} + \text{Term (B)} + \text{Term (C)}\label{eq: proof combine}\notag\\
        &\qquad \leq \max_{R\in\mathcal{R}}\Bigg\{\mathbb{E}_{\substack{x\sim d_0,y_1\sim \pi^\star(\cdot|x, g^\star),\\y_0\sim \pi_{\text{sft}}(\cdot|x, g^\star)}}\Big[\big(R^{\star}(x,y_1, g^\star) -R^{\star}(x,y_0, g^\star)\big) - \big(R(x,y_1, g^\star) -R(x,y_0, g^\star)\big)\Big]\notag\\
&\qquad\qquad+\eta^{-1}\cdot\Big(\cL(R^{\star},\bar{{\cD}}) -  \cL(R,\bar{{\cD}})\Big)\Bigg\}\notag\\
        &\qquad\qquad + \beta \cdot \mathbb{E}_{x\sim d_0}\Big[\mathrm{KL}\big(\pi^\star(\cdot|x, g^\star)\|\pi_{\mathrm{ref}}(\cdot|x, g^\star)\big) - \mathrm{KL}\big(\widehat{\pi}(\cdot|x, g^\star)\|\pi_{\mathrm{ref}}(\cdot|x, g^\star)\big) \Big].
    \end{align}
   Next, we upper bound the right-hand side of \eqref{eq: proof combine} by relating the negative log-likelihood loss difference term to the reward difference term. Recall the definition of the goal-conditioned preference model $\mathbb{P}_R$ in \eqref{eq: bt model}.
    Applying Lemma~\ref{lem: concentration} to give an upper bound of the difference of the negative log-likelihood loss and setting $\varepsilon = (6\cdot(1+e^B)\cdot N)^{-1}$, it holds  with probability at least $1-\delta$ and for any reward model $R\in\mathcal{R}$ that
    \begin{align}
        &\cL(R^{\star},\bar{{\cD}}) -  \cL(R,\bar{{\cD}})\notag \\
        &\qquad \leq -2\cdot \mathbb{E}_{(x,y_1,y_0,g)\sim \mu_{\bar{\cD}}}\Big[D_{\mathrm{Hellinger}}^2\big(\mathbb{P}_{R^{\star}}(\cdot|x,y_1,y_0,g)\|\mathbb{P}_{R}(\cdot|x,y_1,y_0,g)\big)\Big]\notag\\
        &\qquad\qquad + \frac{3}{N}\cdot\log\left(\frac{\cN_{\varepsilon}(\cR,\|\cdot\|_{\infty})}{\delta}\right),
    \end{align}
    where $\cN_{\varepsilon}(\cR,\|\cdot\|_{\infty})$ denotes the $\varepsilon$-covering number \citep{zhou2002covering} of the reward model class $\mathcal{R}$.
    By the relationship between the  Hellinger distance and TV distance, we have
    \begin{align*}
D_{\mathrm{Hellinger}}^2\big(\mathbb{P}_{R^{\star}}(\cdot|x,y_1,y_0,g)\|\mathbb{P}_{R}(\cdot|x,y_1,y_0,g)\big) \geq D_{\mathrm{TV}}^2\big(\mathbb{P}_{R^{\star}}(\cdot|x,y_1,y_0,g)\|\mathbb{P}_{R}(\cdot|x,y_1,y_0,g)\big),
    \end{align*}
    By the definition of the goal-conditioned preference model $\mathbb{P}_R$ in \eqref{eq: bt model}, we have 
    \begin{align}
        &D_{\mathrm{TV}}\big(\mathbb{P}_{R^{\star}}(\cdot|x,y_1,y_0,g)\|\mathbb{P}_{R}(\cdot|x,y_1,y_0,g)\big)\notag\\ &\qquad= \frac{1}{2}\cdot \Big|\sigma\big(R^{\star}(x,y_1, g^\star) -R^{\star}(x,y_0, g^\star)\big) - \sigma\big(R(x,y_1, g^\star) -R(x,y_0, g^\star)\big)\Big|\notag\\
        &\qquad\qquad + \frac{1}{2}\cdot \Big|\sigma\big(R^{\star}(x,y_0, g^\star) -R^{\star}(x,y_1, g^\star)\big) - \sigma\big(R(x,y_0, g^\star) -R(x,y_1, g^\star)\big)\Big|\notag\\
        &\qquad = \Big|\sigma\big(R^{\star}(x,y_1, g^\star) -R^{\star}(x,y_0, g^\star)\big) - \sigma\big(R(x,y_1, g^\star) -R(x,y_0, g^\star)\big)\Big|,\label{eq: tv calculation}
    \end{align}
    where the second equality uses the fact that $\sigma(-z) = 1-\sigma(z)$. 
    Applying Lemma~\ref{lem: sigmoid} and the condition that $R(x,y,g)\in[B/2,B/2]$ for any $(x,y,R,g)\in\cX\times\cA\times\mathcal{R}\times\mathcal{G}$ in Assumption~\ref{as: regularity}, we have 
    \begin{align}
        &\Big|\sigma\big(R^{\star}(x,y_1, g^\star) -R^{\star}(x,y_0, g^\star)\big) - \sigma\big(R(x,y_1, g^\star) -R(x,y_0, g^\star)\big)\Big|\notag\\
        &\qquad \geq \kappa \cdot \Big|\big(R^{\star}(x,y_1, g^\star) -R^{\star}(x,y_0, g^\star)\big) - \big(R(x,y_1, g^\star) -R(x,y_0, g^\star)\big)\Big|,
    \end{align}
    where $\kappa = 1/(1+\exp( B))^2$.
    Therefore, we bound the left-hand side of \eqref{eq: mle difference} as 
    \begin{align}
        &\cL(R^{\star},\bar{{\cD}}) -  \cL(R,\bar{{\cD}})\notag\\
        &\qquad \leq -2\kappa^2 \cdot \mathbb{E}_{(x,y_1,y_0,g)\sim \mu_{\bar{\cD}}}\bigg[\Big|\big(R^{\star}(x,y_1, g) -R^{\star}(x,y_0, g)\big) - \big(R(x,y_1, g) -R(x,y_0, g)\big)\Big|^2\bigg]\notag\\
        &\qquad \qquad + \frac{3}{N}\cdot\log\left(\frac{\cN_{\varepsilon}(\cR,\|\cdot\|_{\infty})}{\delta}\right).\label{eq: mle difference}
    \end{align}
    Meanwhile, the reward difference term in \eqref{eq: proof combine}, which is evaluated on responses sampled from $\pi^\star$ and $\pi_{\text{sft}}$, can be related to the reward difference evaluated on the data distribution $\mu_{\bar{\cD}}$ via Assumption~\ref{as: coverage} as follows,
    \allowdisplaybreaks
    \begin{align}
        &\mathbb{E}_{x\sim d_0,y_1\sim \pi^\star(\cdot|x, g^\star),y_0\sim \pi_{\text{sft}}(\cdot|x, g^\star)}\Big[\big(R^{\star}(x,y_1, g^\star) -R^{\star}(x,y_0, g^\star)\big) - \big(R(x,y_1, g^\star) -R(x,y_0, g^\star)\big)\Big]\notag\label{eq: reward difference}\notag\\
        &\quad \leq C_{\mu_{\bar{\cD}}}(\cR; \pi^\star, \pi_{\text{sft}})  \sqrt{\mathbb{E}_{(x,y_1,y_0,g)\sim \mu_{\bar{\cD}}}\left[\Big|\big(R^{\star}(x,y_1, g) -R^{\star}(x,y_0, g)\big) - \big(R(x,y_1, g) -R(x,y_0, g)\big)\Big|^2\right]}.
    \end{align}
    Combining \eqref{eq: mle difference}, \eqref{eq: reward difference}, and \eqref{eq: proof combine} and defining 
    \begin{align}
        \Delta_{R}:= \sqrt{\mathbb{E}_{(x,y_1,y_0,g)\sim \mu_{\bar{\cD}}}\left[\Big|\big(R^{\star}(x,y_1, g) -R^{\star}(x,y_0, g)\big) - \big(R(x,y_1, g) -R(x,y_0, g)\big)\Big|^2\right]},
    \end{align}
    we obtain 
    \begin{align}
        \mathrm{Gap}^{\pi^\star}(\widehat{\pi}) &\leq \max_{R\in\mathcal{R}}\Big\{C_{\mu_{\bar{\cD}}}(\cR; \pi^\star, \pi_{\text{sft}})\cdot \Delta_{R} - 2\eta^{-1}\kappa^2\cdot \Delta_{R}^2 \Big\}+  \frac{3}{\eta N}\cdot\log\left(\frac{\cN_{\varepsilon}(\cR,\|\cdot\|_{\infty})}{\delta}\right)\notag\\
        &\qquad + \beta \cdot \mathbb{E}_{x\sim d_0}\Big[\mathrm{KL}\big(\pi^\star(\cdot|x, g^\star)\|\pi_{\mathrm{ref}}(\cdot|x, g^\star)\big) - \mathrm{KL}\big(\widehat{\pi}(\cdot|x, g^\star)\|\pi_{\mathrm{ref}}(\cdot|x, g^\star)\big) \Big]\notag\\
        &\leq \frac{\big(C_{\mu_{\bar{\cD}}}(\cR; \pi^\star, \pi_{\text{sft}})\big)^2\eta}{8\kappa^2 } + \frac{3}{\eta N}\cdot \log\left(\frac{\cN_{\varepsilon}(\cR,\|\cdot\|_{\infty})}{\delta}\right)\notag\\
        &\qquad + \beta \cdot \mathbb{E}_{x\sim d_0}\Big[\mathrm{KL}\big(\pi^\star(\cdot|x, g^\star)\|\pi_{\mathrm{ref}}(\cdot|x, g^\star)\big) \Big],
    \end{align}
    where the second inequality uses the fact that $az-bz^2\leq a^2/(4b)$ for any $z\in\mathbb{R}$ and  KL-divergence is non-negative.
    As a result, selecting $\varepsilon = (6\cdot(1+e^B)\cdot N)^{-1}$ and
    \begin{align}
        \eta = 2\sqrt{6}\cdot\sqrt{\frac{\log\left(\cN_{\varepsilon}(\cR,\|\cdot\|_{\infty})/\delta\right)}{N}},\quad \beta = \frac{1}{\sqrt{N}},\quad \kappa = \frac{1}{(1+\exp(B))^2},
    \end{align}
    we prove that with probability at least $1-\delta$ that
    \begin{align}
         &\mathrm{Gap}^{\pi^\star}(\widehat{\pi}) \leq \sqrt{\frac{1}{N}}\cdot\bigg\{\frac{\sqrt{6}}{4} \big(1+\exp(B)\big)^2 \big(\big(C_{\mu_{\bar{\cD}}}(\cR; \pi^\star, \pi_\text{sft})\big)^2+1\big) \iota \notag\\&\qquad\qquad\qquad+  \mathbb{E}_{x\sim d_0}\Big[\mathrm{KL}\big(\pi^\star(\cdot|x, g^\star)\|\pi_{\mathrm{ref}}(\cdot|x, g^\star)\big) \Big]\bigg\},\label{eq:end_eq}
    \end{align}
    where we denote $\iota =  \sqrt{\log\left(\cN_{\varepsilon}(\cR,\|\cdot\|_{\infty})/\delta\right)}$. Combining Theorem \ref{thm:minimax}, \eqref{claim:maxmin}, and \eqref{eq:end_eq}, we
conclude the proof of Theorem \ref{thm:main}.

\subsection{Technical Lemmas}

\begin{lemma}[Uniform concentration]\label{lem: concentration}
    Consider the negative log-likelihood loss in \eqref{eq:mle-loss} and define the approximation error as $\varepsilon = (6\cdot(1+e^B)\cdot N)^{-1}$,  where we assume that $R(x,y,g)\in[-B/2, B/2]$ for any $(R,x,y,g)\in\mathcal{R}\times\mathcal{X}\times\mathcal
    Y\times\mathcal{G}$.
 Suppose that the reward model class $\cR$ has a finite $\varepsilon$-covering number $\cN_{\varepsilon}(\cR,\|\cdot\|_{\infty})<\infty$.
    Then for any $\delta<1/e$ it holds with probability at least $1-\delta$ that 
\begin{align}
    &\cL(R^{\star},\bar{\cD}) -  \cL(R,\bar{\cD}) \\
        &\qquad \leq -2\cdot \mathbb{E}_{(x,y_1,y_0,g)\sim \mu_{\bar{\cD}}}\Big[D_{\mathrm{Hellinger}}^2\big(\mathbb{P}_{R^{\star}}(\cdot|x,y_1,y_0,g)\|\mathbb{P}_{R}(\cdot|x,y_1,y_0,g)\big)\Big]\notag\\
        &\qquad \qquad + \frac{3}{N}\cdot\log\left(\frac{\cN_{\varepsilon}(\cR,\|\cdot\|_{\infty})}{\delta}\right).
\end{align}
\end{lemma}
\begin{proof}
    See the proof of Lemma D.1 in \citet{liu2024provably}, where we use the fact that $(x,g)$ follows a fixed distribution.
\end{proof}

\begin{lemma}[Difference of Sigmoid functions]\label{lem: sigmoid}
    For any real numbers $z_1,z_2\in[-B/2,B/2]$, it holds that 
    \begin{align}
        \kappa\cdot |z_1-z_2|\leq\left|\sigma(z_1) - \sigma(z_2)\right|\leq |z_1-z_2|,
    \end{align}
    where the constant $\kappa = 1/(1+\exp(B))^2$.
\end{lemma}
\begin{proof}
    See the proof of Lemma D.2 in \citet{liu2024provably}.
\end{proof}
\color{black}
\section{Experiment Details}
\subsection{Setup}
\label{app_setup}
We use the following prompt during training. Here, the reward values are the quality scores given by the judge models that exist in the preference dataset. 
The prompt is set as the system prompt whenever the LLM supports, such as Qwen2-7B-Instruct and Llama-3.1-8B-Instruct, and it is prefixed before the original prompt when the LLM doesn't support system prompting, such as Mistral-7B-Instruct-v0.3 and Gemma-2-9B-It.
\begin{Verbatim}
You are an assistant that generates responses for the instruction while implicitly
achieving the following target score (on a scale of 1-10, where 1 is lowest and 10 
is highest):
Overall score: {reward_value}.
\end{Verbatim}

At inference time, we use almost the same prompt, except that the goal score is the highest one, i.e., the overall score is $10$.
\vspace{0.25cm}
\begin{Verbatim}
You are an assistant that generates responses for the instruction while implicitly
achieving the following target score (on a scale of 1-10, where 1 is lowest and 10
is highest):
Overall score: 10.
\end{Verbatim}

In our experiments using UltraFeedback, we directly leverage the LLM-as-Judge scores provided by GPT-4 in the dataset, which range from $1$ to $10$. For our method that is applied to on-policy data ranked by external reward models, including PairRM and ArmoRM, we apply linear transformations to normalize the resulting reward scores, ensuring they are scaled within the same $1$ to $10$ range.

For hyperparameters, we tune the KL regularization coefficient $\beta$ within $[0.001, 0.01, 0.1]$ and batch size within $[64, 128, 256]$. We find that $\beta=0.01$ and a $256$ batch size yield the overall best performance for DPO across models. Our method uses the same hyperparameters as DPO. Besides, we adopt the AdamW optimizer \citep{loshchilov2017decoupled}, with a learning rate of $5e-7$ and a warmup ratio of $0.1$. Furthermore, we observe that for models such as Qwen2-7B-Instruct and Gemma-2-9B-It on UltraFeedback, as well as Llama-3-8B-Instruct on on-policy data, both DPO and our proposed method yield improved performance when employing the conservative DPO (cDPO) technique \citep{mitchell2023note}. Consequently, for these models, we set the label smoothing hyperparameter from the Alignment Handbook \citep{Tunstall_The_Alignment_Handbook} to $0.3$, while keeping it at $0$ for the remaining models.

\subsection{Full Results}
\label{sec_full}
In Table \ref{tab_full}, we present the full results on instruction-following benchmarks, which correspond to the performance illustrated in Figure \ref{lmasjudge} in the main text.

\begin{table}[H]
\centering
\begin{tabular}{l|ccc|ccc|cc}
\hline
& \multicolumn{3}{c|}{AlpacaEval 2.0} & \multicolumn{3}{c|}{MT-Bench} & \multicolumn{2}{c}{Arena-Hard-Auto}  \\
& \begin{tabular}{@{}c@{}}LC WR\end{tabular} & \begin{tabular}{@{}c@{}}WR \end{tabular} & \begin{tabular}{@{}c@{}}Avg. Len. \end{tabular} & Avg. & \begin{tabular}{@{}c@{}}1st\end{tabular} & \begin{tabular}{@{}c@{}}2nd\end{tabular}& \begin{tabular}{@{}c@{}}Score \end{tabular}  & \begin{tabular}{@{}c@{}}Avg. Len. \end{tabular}\\ \hline
Mistral-7B-Instruct-v0.3 & 19.65  & 15.40  & 1503 & 7.67  & 8.00  & 7.34 & 17.0 & 494 \\
+DPO (UltraFeedback) & 18.76  & 16.93  & 1643 & 7.66  &  7.92& \textbf{7.40}  & 17.6 & 504 \\
+DPO (Reward-Augmented)  &  \textbf{25.99} & \textbf{28.36}  & 2270 & \textbf{7.69} &  \textbf{8.02}  & 7.36 & \textbf{18.3} & 883 \\
\hline
Qwen2-7B-Instruct & 20.93  & 18.22  & 1788 & 7.90  & 8.23  & 7.56 & 24.3 & 617 \\
+DPO (UltraFeedback) & 21.46  & 19.35 & 1797 & 8.33 &  8.72 & 7.93  & 21.9 & 553 \\
+DPO (Reward-Augmented) & \textbf{31.17} & \textbf{27.58} & 1789 & \textbf{8.47} &  \textbf{8.93}  & \textbf{7.97} & \textbf{30.1} & 644 \\
\hline
Llama-3.1-8B-Instruct & 24.79& 27.38 & 2081 & 8.44 & 8.99 & 7.90 & 26.9 & 831 \\
+DPO (UltraFeedback) &28.67 & 30.21& 2053 & 8.47 & \textbf{9.01} & 7.93 & 33.0 & 1070 \\
+DPO (Reward-Augmented) & \textbf{31.20} & \textbf{35.93} & 2006 & \textbf{8.47} & 8.91 & \textbf{8.03} & \textbf{34.4} & 824 \\
\hline
Gemma-2-9B-It & 49.20 & 37.58 & 1572 & 8.54 & 8.81 & 8.28 &42.8  & 541\\
+DPO (UltraFeedback) & 50.70 & 35.02 & 1464 & 8.54 & 8.70 & \textbf{8.37} & 35.8 & 456 \\
+DPO (Reward-Augmented) & \textbf{59.27} & \textbf{54.56} & 1872 & \textbf{8.59} & \textbf{8.93} & 8.25 & \textbf{43.9}  & 611 \\
\hline
SPPO & 55.60 & 49.61 & 1822 & 8.40 & 8.53 & 8.26 & 47.6 & 578 \\
+DPO (UltraFeedback) & 52.75 & 40.58 & 1544 & 8.41 & 8.78 & 8.04 & 40.4 & 457\\
+DPO (Reward-Augmented) & \textbf{60.97} & \textbf{66.41} & 2543 & \textbf{8.73} & \textbf{9.06} & \textbf{8.41} & \textbf{49.0} & 761 \\
\hline
\end{tabular}
\caption{Results of the DPO models fine-tuned on UltraFeedback and on reward-augmented UltraFeedback. We evaluate on the instruction-following benchmarks including AlpacaEval 2.0, MT-Bench, and Arena-Hard-Auto.}
\label{tab_full}
\end{table}

\color{black}
We also provide the full comparison results with reward-augmented methods in Table \ref{full_com}, which corresponds to Figure \ref{fig_baselines} in the main text.
\begin{table}[H]
\centering
\begin{tabular}{l|ccccccccc}
\hline
\begin{tabular}{@{}c@{}} \\ \end{tabular} 
& \textcolor{black}{Zephyr-SFT} 
& \textcolor{black}{DPO} 
& \textcolor{black}{DPA} 
& \textcolor{black}{SteerLM} 
& \textcolor{black}{NCA-P} 
& \textcolor{black}{NCA-R} 
& \textcolor{black}{INCA-P} 
& \textcolor{black}{INCA-R} 
& \textcolor{black}{Ours} 
 \\ \hline
\textcolor{black}{LC Win Rate} 
& \textcolor{black}{6.21} 
& \textcolor{black}{11.60} 
& \textcolor{black}{11.13} 
& \textcolor{black}{-} 
& \textcolor{black}{11.50}
& \textcolor{black}{12.87} 
& \textcolor{black}{13.68} 
& \textcolor{black}{14.83} 
& \textcolor{black}{16.66}  \\ 
\textcolor{black}{Win Rate} 
& \textcolor{black}{3.94} 
& \textcolor{black}{8.58} 
& \textcolor{black}{10.58} 
& \textcolor{black}{8.21} 
& \textcolor{black}{8.43}
& \textcolor{black}{9.56} 
& \textcolor{black}{11.00} 
& \textcolor{black}{11.34} 
& \textcolor{black}{13.37}  \\ 
\textcolor{black}{Avg. Len.} 
& \textcolor{black}{893} 
& \textcolor{black}{1240} 
& \textcolor{black}{1671} 
& \textcolor{black}{1585} 
& \textcolor{black}{1287} 
& \textcolor{black}{1364} 
& \textcolor{black}{1449} 
& \textcolor{black}{1338} 
& \textcolor{black}{1812} \\ \hline
\end{tabular}
\vspace{-0.25cm}
\caption{Full comparison results with reward-augmented methods.}
\label{full_com}
\end{table}

\subsection{More Ablations}
\label{app_more_abl}
\paragraph{Impact of the Accuracy of AI Feedback.} We consider the 19.8k prompts from a 1/3 subset of UltraFeedback following the setup from Snorkel \citep{snorkelaipair}. Five on-policy responses are first generated from Llama-3-8B-Instruct. An external reward model is followed to rank these responses. We choose the best and worst responses as the chosen and rejected ones. DPO is then performed on the resulting preference pairs and the reward-augmented pairs. To ablate how our method will be impacted by the accuracy of AI feedback, we experiment with two reward models as the ranker: PairRM \citep{jiang2023llm} and ArmoRM \citep{wang2024interpretable}. PairRM is a small-sized (0.4B) pairwise reward model, while ArmoRM is a 8B model that is state-of-the-art on RewardBench \citep{lambert2024rewardbench} and much stronger than PairRM. We implement a variant (denoted as RA+) of the proposed reward augmentation method that only conditions on the goal rewards of the chosen responses, not those of the rejected ones, leading to same-sized datasets.

\begin{table*}[htb]
\centering
\begin{tabular}{l|c|ccc|cc}
\hline
 & Llama-3- & \multicolumn{3}{c|}{PairRM (0.4B)} & \multicolumn{2}{c}{ArmoRM (8B)} \\
 & 8B-Instruct & \begin{tabular}{@{}c@{}}DPO (UF)\end{tabular} & \begin{tabular}{@{}c@{}}DPO (RA+) \end{tabular} & \begin{tabular}{@{}c@{}}DPO (RA)\end{tabular} & DPO (UF) & \begin{tabular}{@{}c@{}}DPO (RA+)\end{tabular} \\ \hline
LC WR  & 22.92 & 41.76 & 44.72 & 48.20 & 42.32  & \textbf{48.73}\\
WR  & 23.15 & 45.79 & 44.70 & \textbf{53.17}  & 42.79 & 45.36 \\
 \hline
\end{tabular}
\setlength{\belowcaptionskip}{-10pt}
\caption{Ablation on the impact of AI feedback quality on the AlpacaEval 2.0 benchmark.\vspace{-0.2cm}}
\label{tab_aif}
\end{table*}

The results in Table \ref{tab_aif} demonstrate that training on augmented data conditioned on both chosen and rejected rewards is necessary for PairRM feedback, while relabeling with only the chosen rewards is sufficient to achieve strong performance for ArmoRM feedback. This aligns with our motivation outlined in Section \ref{sec_lim}: in noisy preference data, rejected responses may actually be of high quality, unlearning which can degrade performance. Similarly, low-quality chosen responses may also be reinforced. This issue does not arise with strong reward models that provide accurate preferences.

\paragraph{More Comparisons with Baselines}
In Table \ref{tab_pairrm}, we compare our method and various baselines under the same setting on the AlpacaEval 2.0 benchmark, including SLiC-HF \citep{zhao2023slic}, ORPO \citep{hong2024orpo}, CPO \citep{xu2024contrastive}, RRHF \citep{yuan2024rrhf}, KTO \citep{ethayarajh2024kto}, IPO \citep{azar2023general}, R-DPO \citep{park2024disentangling}, and SimPO \citep{meng2024simpo}, where the results are from \cite{meng2024simpo}, as well as the RPO \citep{adler2024nemotron} baseline that we implement. Our method outperforms the above algorithms by a considerable margin.

\begin{table*}[htb]
\centering
\begin{tabular}{l|c|c|c|c|c|c|c|c|c|c}
\hline
 & SLiC-HF & ORPO & CPO & RRHF & KTO & IPO & RPO & R-DPO & SimPO & Ours \\ \hline
LC WR  & 26.9 & 28.5 & 28.9 & 31.3 & 33.1 & 35.6 & 40.8 & 41.1  & 44.7  & \textbf{48.2}\\
WR  & 27.5 & 27.4 & 32.2 & 28.4 & 31.8 & 35.6 & 41.7 & 37.8 & 40.5 & \textbf{53.2} \\
 \hline
\end{tabular}
\vspace{-0.05cm}
\setlength{\belowcaptionskip}{-10pt}
\caption{Comparison between our method, i.e., Llama-3-8B-Instruct+DPO (RA) and baselines fine-tuned on the same model and on-policy data ranked by PairRM.}
\label{tab_pairrm}
\end{table*}

In addition to the comparison with RPO from \cite{adler2024nemotron} in Table \ref{tab_pairrm}, we also report the performance of RPO fine-tuned on additional models including Qwen2-7B-Instruct and Gemma2-9B-It. As shown in Table \ref{tab_rpo}, the implemented RPO is outperformed by our method across these models.

\begin{table}[H]
\centering
\begin{tabular}{l|ccc}
\hline
&\begin{tabular}{@{}c@{}}LC Win Rate\end{tabular} & \begin{tabular}{@{}c@{}}Win Rate \end{tabular} & \begin{tabular}{@{}c@{}}Avg. Len.\end{tabular} \\ \hline
Qwen+RPO & 20.29 & 17.34  & 1704 \\
Qwen+DPO (RA) & \textbf{31.17} & \textbf{27.58}  & 1789 \\ \hline
Gemma+RPO & 43.14 & 30.93  & 1413 \\
Gemma+DPO (RA) & \textbf{59.27} & \textbf{54.56}  & 1872 \\ \hline
\end{tabular}
\captionof{table}{Comparison on AlpacaEval 2.0 between our method and RPO fine-tuned from the Qwen2-7B-Instruct and Gemma2-9B-It models. Our method consistently outperforms RPO across these fine-tuned models.}
\label{tab_rpo}
\end{table}

\paragraph{Benefits of Learning from High-Quality Rejected Responses.} Using the UltraFeedback dataset, we construct two reward-augmented preference datasets by filtering out augmented data based on rejected responses with low and high reward values, respectively. Compared to our method, these datasets isolate the impact of excluding low- and high-reward rejected responses as goals. The evaluation results on AlpacaEval 2.0 are presented in Table \ref{ablation_rej}. Learning from rejected high-reward samples demonstrates superior performance compared to the approach that excludes these samples.
\begin{table}[H]
\centering
\begin{tabular}{l|ccccc}
\hline
\begin{tabular}{@{}c@{}} \\ \end{tabular} 
& \textcolor{black}{Qwen2-7B-It} 
& \textcolor{black}{+DPO (UF)} 
& \textcolor{black}{+DPO (RA)} 
& \textcolor{black}{\begin{tabular}{@{}c@{}}+DPO (RA \\ filter high)\end{tabular}} 
& \textcolor{black}{\begin{tabular}{@{}c@{}}+DPO (RA \\ filter low)\end{tabular}} \\ \hline
\textcolor{black}{LC Win Rate} 
& \textcolor{black}{20.93} 
& \textcolor{black}{21.46} 
& \textcolor{black}{31.17} 
& \textcolor{black}{29.36} 
& \textcolor{black}{\textbf{31.81}} \\ 
\textcolor{black}{Win Rate} 
& \textcolor{black}{18.22} 
& \textcolor{black}{19.35} 
& \textcolor{black}{\textbf{27.58}} 
& \textcolor{black}{27.04} 
& \textcolor{black}{27.28} \\ \hline
\end{tabular}
\vspace{-0.25cm}
\caption{Ablation on the benefits of learning from high-quality rejected responses.}
\label{ablation_rej}
\end{table}

\paragraph{Impact of the Reward Scale.} For the UltraFeedback dataset that contains response rewards in the range of 1-10, we relabel them to be in the range of 1-5 and 1-100 with linear transformation. Our method followed by DPO is then applied on these different scaled datasets. The results are shown in Table \ref{ablation_scale}. It can be observed that our method is robust to the reward scales. Since our main experiments use the default 1-10 scale as in UltraFeedback, it is likely that the performance can be further boosted, e.g., by adopting the 1-100 scale.
\begin{table}[H]
\centering
\begin{tabular}{l|ccccc}
\hline
\begin{tabular}{@{}c@{}} \\ \end{tabular} 
& \textcolor{black}{Qwen2-7B-It} 
& \textcolor{black}{+DPO (UF)} 
& \textcolor{black}{+DPO (RA, 5)} 
& \textcolor{black}{+DPO (RA 10)} 
& \textcolor{black}{+DPO (RA 100)} 
 \\ \hline
\textcolor{black}{LC Win Rate} 
& \textcolor{black}{20.93} 
& \textcolor{black}{21.46} 
& \textcolor{black}{29.85} 
& \textcolor{black}{31.17} 
& \textcolor{black}{\textbf{31.81}} \\ 
\textcolor{black}{Win Rate} 
& \textcolor{black}{18.22} 
& \textcolor{black}{19.35} 
& \textcolor{black}{26.12} 
& \textcolor{black}{27.58} 
& \textcolor{black}{\textbf{27.96}} \\ \hline
\end{tabular}
\vspace{-0.25cm}
\caption{Ablation on the impact of the reward scale demonstrates the robustness of our method.}
\label{ablation_scale}
\end{table}
\end{document}